\def\year{2015}
\documentclass[letterpaper]{article}
\usepackage{aaai}
\usepackage{times}
\usepackage{helvet}
\usepackage{courier}

\usepackage{url}
\usepackage{epsfig}
\usepackage{epstopdf}
\usepackage{amssymb}
\usepackage{amsmath}
\usepackage{amsthm}
\newtheorem{theorem}{Theorem}
\newtheorem{lemma}[theorem]{Lemma}
\newtheorem{proposition}[theorem]{Proposition}

\DeclareMathOperator*{\sign}{sign}

\usepackage{algorithm}
\usepackage{algorithmic}
\usepackage{graphicx} 
\usepackage{subfigure}
\usepackage{url}
\usepackage{booktabs}
\usepackage{color}
\usepackage{enumitem}
\usepackage{array}

\begin{document}
%
\title{Stable Feature Selection from Brain sMRI}
\author{Bo Xin$^*$, Lingjing Hu$^\dag$, Yizhou Wang$^*$ and Wen Gao$^*$\\
$^*$National Engineering Laboratory for Video Technology, Key Laboratory of Machine Perception,\\
School of EECS, Peking University, Beijing, 100871, China \\
$^\dag$Yanjing Medical College, Capital Medical University, Beijing, 101300, China
}
\maketitle
\begin{abstract}
\begin{quote}
Neuroimage analysis usually involves learning thousands or even millions of variables using only a limited number of samples. In this regard, sparse models, e.g. the lasso, are applied to select the optimal features and achieve high diagnosis accuracy. The lasso, however, usually results in independent unstable features. Stability, a manifest of reproducibility of statistical results subject to reasonable perturbations to data and the model \cite{Yu13}, is an important focus in statistics, especially in the analysis of high dimensional data. In this paper, we explore a {\em nonnegative generalized fused lasso model} for stable feature selection in the diagnosis of Alzheimer's disease. In addition to sparsity, our model incorporates two important pathological priors: the spatial cohesion of lesion voxels and the positive correlation between the features and the disease labels. To optimize the model, we propose an efficient algorithm by proving a novel link between total variation and fast network flow algorithms via conic duality. Experiments show that the proposed nonnegative model performs much better in exploring the intrinsic structure of data via selecting stable features compared with other state-of-the-arts.
\end{quote}
\end{abstract}

\section{Introduction}
\label{sec:intro}

Neuroimage analysis is challenging due to its high feature dimensionality and data scarcity. Sparse models such as the lasso \cite{Tib96} have gained great reputation in statistics and machine learning, and they have been applied to the analysis of such high dimensional data by exploiting the sparsity property in the absence of abundant data. As a major result, automatic selection of relevant variables/features by such sparse formulation achieves promising performance. For example, in \cite{liu2012ensemble}, the lasso model was applied to the diagnosis of Alzheimer's disease (AD) and showed better performance than the support vector machine (SVM), which is one of the state-of-the-arts in brain image classification.
However, in statistics, it is known that the lasso does not always provide interpretable results because of its instability \cite{Yu13}. ``Stability" here means the reproducibility of statistical results subject to reasonable perturbations to data and the model. (These perturbations include the often used Jacknife, bootstrap and cross-validation.) This unstable behavior of the lasso model is critical in high dimensional data analysis. The resulting irreproducibility of the feature selection are especially undesirable in neuroimage analysis/diagnosis. However, unlike the problems such as registration and classification, the {\em stability issue} of feature selection is much less studied in this field.

In this paper we propose a model to induce more stable feature selection from high dimensional brain structural Magnetic Resonance Imaging (sMRI) images. Besides sparsity, the proposed model harnesses two important additional pathological priors in brain sMRI: (i) the spatial cohesion of lesion voxels (via inducing fusion terms) and (ii) the positive correlation between the features and the disease labels.
The correlation prior is based on the observation that in many brain image analysis problems (such as AD, frontotemporal dementia, corticobasal degeneration, etc), there exist strong correlations between the features and the labels. For example, gray matter of AD is degenerated/atrophied. Therefore, the gray matter values (indicating the volume) are positively correlated with the cognitive scores or disease labels \{-1,1\}. That is, the less gray matter, the lower the cognitive score. Accordingly, we propose nonnegative constraints on the variables to enforce the prior and name the model as ``non-negative Generalized Fused Lasso'' ($n^2$GFL). It extends the popular generalized fused lasso and enables it to explore the intrinsic structure of data via selecting stable features. To measure feature stability, we introduce the ``Estimation Stability" recently proposed in \cite{Yu13} and the (multi-set) Dice coefficient \cite{dice1945measures}. Experiments demonstrate that compared with existing models, our model selects much more stable (and pathological-prior consistent) voxels. It is worth mentioning that the non-negativeness per se is a very important prior of many practical problems, e.g. \cite{lee1999learning}. Although $n^2$GFL is proposed to solve the diagnosis of AD in this work, the model can be applied to more general problems.

Incorporating these priors makes the problem novel w.r.t the lasso or generalized fused lasso from an optimization standpoint. Although off-the-shelf convex solvers such as CVX \cite{cvx} can be applied to solve the optimization, it hardly scales to high-dimensional problems in feasible time. In this regard, we propose an efficient algorithm that solves the $n^2$GFL problem exactly. We generalize the proximal gradient methods (such as FISTA) \cite{beck2009fast} to solve our constrained optimization and prove its convergence. We then show that by using an element-wise post-processing, the resulting proximal operator can be reduced to the total variation (TV) problem. It is known that TV can be solved by parametric flow algorithms \cite{chambolle2009total,bo2014gfused}. In the present study, we provide a novel equivalence via conic duality, which gives us a minimum quadratic cost flow formulation \cite{hochbaum1995strongly}. Fast flow algorithms (including parametric flow) are then easily applied. In practice, our algorithm runs hundreds of times faster than CVX at the same precision and can scale to high-dimensional problems.

{\bf Related work.} In addition to sparsity, people leverage underlying data structures and introduce stronger priors such as the structured sparsity \cite{jacob2009group} to increase model stability. However, for voxel-based sMRI data analysis, handcrafted grouping of the voxels or sub-structures may not coincide with various pathological topology priors. Consequently, group lasso (with overlap) \cite{jacob2009group,jenatton2012multiscale,rao2013sparse} is not an ideal model to the problem. In contrast, the graph-based structured sparse models adapt better to such a situation. The most popular one is referred here as LapL\footnote{Although different names are given in e.g. \cite{ng2011generalized,grosenick2013interpretable}, they are in fact fundamentally applying the graph Laplacian smoothing.}, which adopts $l_2$ norm regularization of neighborhood variable difference (e.g. \cite{ng2011generalized,grosenick2013interpretable}). However, as we will show in the experiments, these models select many more features than necessary. Very recently, generalized fused lasso or total variation has been successful applied to brain image analysis problems inducing the $l_1$ difference \cite{gramfort2013identifying,bo2014gfused}. In the experiments, we show that by including an extra nonnegative constraint, the features selected by our model is much more stable than that of such unconstrained models.
A very recent work \cite{avants2014sparse} also explored this positive correlation (partially supporting our assumption), but the problem formulation was quite different: neither structural assumption was considered, nor the stability of feature selection was discussed.
From the optimization standpoint, the applied framework is similar to that of \cite{bo2014gfused} but two key differences exist: (1) the FISTA and soft-thresholding process applied in \cite{bo2014gfused} do not generalize to constrained optimization problems, we show important modifications and provide theoretical proof; (2) we propose a novel understanding of TV's relation with flow problems via conic duality and prove that the minimum norm point problem of \cite{bo2014gfused} is a special case of our framework.

\section{The Proposed Method}
\label{sec:mets}

\subsection{Nonnegative Generalized Fused Lasso ($n^2$GFL)}
\label{ssec:cgfl}

Let $\{(\mathbf{x}_i,y_i)\}_{i=1}^N$ be a set of samples, where $\mathbf{x}_i \in \mathbb{R}^d$ and $y_i\in\mathbb{R}$ are features and labels, respectively. Also, we denote by $\mathbf{X} \in \mathbb{R}^{d\times N}$ and $\mathbf{y}\in\mathbb{R}^N$ the concatenations of $\mathbf{x}_i$ and $y_i$. Then, we consider the formulation
\begin{equation}
\label{eq:gfl}
\begin{split}
 \min_{\boldsymbol\beta\in\mathbb{R}^d}~ & l(\boldsymbol\beta; \mathbf{X}, \mathbf{y}) + \lambda_1 \sum_{i=1}^{d}{|\beta_i|} + \lambda_2 \sum_{(i,j)\in E}{w_{ij}|\beta_i - \beta_j|}, \\
 ~& ~s.t. ~~\boldsymbol\beta \ge \mathbf{0}
\end{split}
\end{equation}
where $\lambda_1,\lambda_2\geq 0$ are tuning parameters. $l$ is a loss term of variable $\boldsymbol\beta$ (assumed to be convex and smooth). $w_{ij}$s are pre-defined weights.
Here, the variables (e.g. the sMRI voxel parameters in the AD problem) are supposed to have certain underlying structure represented by a graph $G=(V,E)$ with nodes $V$ and edges $E$. Each variable corresponds to a node on the graph.
As mentioned above, in many brain image analysis, there exist strong directional correlations (positive or negative) between the features and the labels, thus we assume $\boldsymbol\beta \ge \mathbf{0}$ (or $\boldsymbol\beta \le \mathbf{0}$).
Due to the $l_1$ penalties on each variable as well as each adjacent pair of variables in \eqref{eq:gfl}, solutions tend to be both sparse and smooth, i.e., adjacent variables tend to be similar and spatially coherent. Also because we have added the nonnegative constraints, the model will not select negatively correlated features as support. In practice, we notice that unconstrained models will systematically select many negatively correlated features. The nonnegative constraints greatly reduce these falsely recovered variables and encourage genuine disease-related features to be selected.

\subsection{Efficient Optimization of $n^2$GFL}
\label{ssec:opti}

The optimization of $n^2$GFL is convex and off-the-shelf solver such as CVX can be applied. However, this solution hardly scales to a problem sized of thousands (mainly due to its choice of general second order frameworks), see Table \ref{tab:eff}. In this regard, we propose certain modifications to scalable first order methods (e.g. accelerated proximal methods \cite{beck2009fast}). This is done by exploring Lagrange multiplier method to deal with the constraints. From the optimization standpoint, these modifications are non-trivial and compose one major contribution of this work.

We first extend the (fast) iterative shrinkage thresholding algorithm (ISTA and FISTA) \cite{beck2009fast} as follows.
\begin{proposition}
\label{prop:cfista}
    Let $\boldsymbol\beta^*$ be the optimal solution to \eqref{eq:gfl} and $\boldsymbol\beta^{k}$ defined as follows
    \begin{equation}
        \label{eq:po}
        \begin{split}
        \boldsymbol\beta^{k+1} = & \min_{\boldsymbol\beta\in\mathbb{R}^d}~ \frac{1}{2} \Vert \boldsymbol\beta - \mathbf{z}^{k} \Vert_2^2
         + \frac{\lambda_1}{L} \sum_{i=1}^{d}{|\beta_i|} +  \\
         & \frac{\lambda_2}{L} \sum_{(i,j)\in E}{w_{ij}|\beta_i - \beta_j|}, ~~~~s.t. ~~ \boldsymbol\beta \ge \mathbf{0},
        \end{split}
    \end{equation}
    where $L>0$ is the Lipschitz constant of $\nabla l(\cdot)$ and $k$ is the iteration number.

    \noindent
    If $\mathbf{z}^{k}$$=$$ \boldsymbol\beta^{k} - \frac{1}{L} \nabla l(\boldsymbol\beta^{k})$, then $F(\boldsymbol\beta_{k}) - F(\boldsymbol\beta^*) \le  \frac{\alpha L \Vert \boldsymbol\beta^0 - \boldsymbol\beta^* \Vert_2^2}{2 k}$, where $F(\cdot)$ is the objective of \eqref{eq:gfl}.
    If $\mathbf{z}^{k} = \mathbf{y}^{k} - \frac{1}{L} \nabla l(\mathbf{y}^{k})$ where $\mathbf{y}^{k}=\boldsymbol\beta^{k}+\alpha^k (\boldsymbol\beta^k - \boldsymbol\beta^{k-1})$ with $\alpha$ controlling the momentum, we have $F(\boldsymbol\beta_{k}) - F(\boldsymbol\beta^*) \le  \frac{2 \alpha L \Vert \boldsymbol\beta_0 - \boldsymbol\beta^* \Vert_2^2}{(k+1)^2}$
\end{proposition}
The proof can be viewed as an instantiation of the convex analysis introduced in \cite{nesterov2004introductory}. 
We provide a rigorous proof in the supplementary file.

Now the key to solve \eqref{eq:gfl} is how efficiently we solve \eqref{eq:po}.
If there were no constraints, \eqref{eq:po} is the fused lasso signal approximation proposed in \cite{friedman2007pathwise}, where it was shown that by utilizing the separability of the $l_1$ norm, an element-wise soft-threshold technique can be applied to remove the sparse term. Since the constraints of \eqref{eq:po} are also separable, we show how \eqref{eq:po} can be further reduced likewise.
\begin{proposition}
\label{prop:totv}
     If we define
     \begin{equation}
     \label{eq:tv}
         \tilde {\boldsymbol{\beta}} = \min_{\boldsymbol\beta\in\mathbb{R}^d}~ \frac{1}{2} \Vert \boldsymbol\beta - \mathbf{z} \Vert_2^2 + \frac{\lambda_2}{L} \sum_{(i,j)\in E}{w_{ij}|\beta_i - \beta_j|},
     \end{equation}
     then the optimal solution to \eqref{eq:po} (denoted as $\boldsymbol\beta^*$) can be achieved by an element-wise post-processing to $\tilde {\boldsymbol{\beta}}$ as follows
     \begin{equation}
            \boldsymbol\beta^* = \max (\sign( \tilde {\boldsymbol{\beta}}) \odot  \max(|  \tilde {\boldsymbol{\beta}}| - \frac{\lambda_1}{L} , \mathbf{0}),\mathbf{0});
     \end{equation}
     where $\odot$ is an element-wise product operator.
\end{proposition}
\begin{proof}
We define $\theta_i = \frac{\lambda_1}{L}$ and $\theta_{ij} = \frac{\lambda_2 w_{ij}}{L}$ respectively for all $i\in V$ and $(i,j)\in E$. We denote $\boldsymbol{\beta}'$ as the optimal solution of  the unconstrained problem of \eqref{eq:po}. According to \cite{friedman2007pathwise},  $\beta'_i = \sign(\tilde\beta_i)\max(|\tilde\beta_i|- \theta_i,0)$. We now consider the nonnegative constraints in \eqref{eq:po}. According to the Karush-Kuhn-Tucker (KKT) conditions, the necessary and sufficient conditions for $\beta_1^*,...\beta_d^*$ are
\begin{equation}
\label{eq:kkt}
\begin{split}
    Lg_i & = (\beta_i-z_i) + \theta_i s_i + \sum_{j:(i,j)\in \mathcal{E}}{\theta_{ij} t_{ij}} - \\
    & \sum_{j:(j,i)\in \mathcal{E}}{\theta_{ij} t_{ji}} - \alpha_i =0,~~~~s.t. ~~~\alpha_i \beta_i = 0,
\end{split}
\end{equation}
where $\alpha_i \ge 0$ are the Lagrange multipliers and $\mathbf{s}$, $\mathbf{t}$ are sub-gradients: $s_i = \sign(\beta_i)$ if $\beta_i\neq 0$ and $s_i \in [-1,1]$ if $\beta_i=0$; $t_{ij} = \sign(\beta_i-\beta_j)$ for $\beta_i \neq \beta_j$ and $t_{ij}\in [-1,1]$ if $\beta_i = \beta_j$.
The objective equation in \eqref{eq:kkt} is to set the derivative of the Lagrange function to zero and the constraint equations are obtained from the complementary slackness condition.
We now consider two cases of $\boldsymbol\beta$:

\noindent
{\bf Case 1 $\beta'_i \ge 0$}: Note that by setting $\alpha_i = 0$, $\beta'_i \ge 0$ satisfies the conditions in \eqref{eq:kkt}, thus $\beta_i^* = \beta'_i = \sign(\tilde\beta_i)\max(|\tilde\beta_i|- \theta_i,0)$ is the solution of \eqref{eq:po}.

\noindent
{\bf Case 2 $\beta'_i < 0 $}: We can set $\beta_i^* = 0$ and $\alpha_i = -\beta'_i > 0$, then we have $\beta'_i = \beta^*_i - \alpha_i$.
\begin{align*}
    &~~~~~~ Lg_i  \\
    &  = (\beta^*_i-z_i) + \theta_i s_i + \sum_{j:(i,j)\in \mathcal{E}}{\theta_{ij} t_{ij}} - \sum_{j:(j,i)\in \mathcal{E}}{\theta_{ij} t_{ji}} - \alpha_i \\
    & = (\beta^*_i - \alpha_i -z_i) + \theta_i s_i + \sum_{j:(i,j)\in \mathcal{E}}{\theta_{ij} t_{ij}} - \sum_{j:(j,i)\in \mathcal{E}}{\theta_{ij} t_{ji}} \\
    & = (\beta'_i -z_i) + \theta_i s_i + \sum_{j:(i,j)\in \mathcal{E}}{\theta_{ij} t_{ij}} - \sum_{j:(j,i)\in \mathcal{E}}{\theta_{ij} t_{ji}}= 0.
\end{align*}
Hence, in summary, we have $$\boldsymbol\beta^* = max(sign( \tilde {\boldsymbol{\beta}}) \odot  max(|  \tilde {\boldsymbol{\beta}}| - \frac{\lambda_1}{L} , \mathbf{0}), \mathbf{0}). \qedhere $$
\end{proof}

Notice that, \eqref{eq:tv} is a (continous) total variation problem, which is known can be efficiently solved by parametric flow algorithms in \cite{chambolle2009total,bo2014gfused}. Here we present a more general perspective of such an equivalence via conic duality, which gives us a natural and novel minimum quadratic cost flow formulation. Fast flow algorithms, such as but not limited to parametric flow \cite{GGT89} etc. are then easily applied. For example, we show that the minimum norm point problem solved by parametric flow in \cite{bo2014gfused} can be viewed as a special case of the proposed dual.

\subsection{Conic Dual to Total Variation}
\label{ssec:cdttv}
To solve \eqref{eq:tv}, we apply generalized inequalities and its corresponding Lagrange duality introduced in \cite{boyd2004convex}. 
Specifically, we first define a set $C$ such that $C$$=$$\{(\boldsymbol\beta,\alpha) \in \mathbb{R}^{d+1} ~|~ \forall (i,j), ~|\beta_i-\beta_j|\le\alpha \}$.
\begin{lemma}
\label{lem:proper}
    The set $C$ is a proper cone.
\end{lemma}
This can be easily shown by checking all the properties required by a proper cone. See the supplementary file for a proof. We now consider the following problem (we keep using $\theta_{ij}=\frac{\lambda_2 w_{ki}}{L}$  for $\frac{\lambda_2 w_{ki}}{L}$  in  \eqref{eq:tv}):
\begin{equation}
\label{prob2}
\begin{split}
    & \min_{\forall (i,j) \in E,  \{\boldsymbol\beta^{ij} \in \mathbb{R}^d, \alpha^{ij} \in \mathbb{R}\}}{\frac{1}{2}\Vert \boldsymbol\beta-\mathbf{z} \Vert_2^2 + \sum_{(i,j)\in E}{\theta_{ij} \alpha^{ij}}} \\
    &~~ s.t. ~~(\boldsymbol\beta^{ij}, \alpha^{ij}) \in C  ~ \text{and} ~\beta^{ij}_k = \begin{cases} \beta_k & k = i, j \\
    0 & else \end{cases}.
\end{split}
\end{equation}
Since $(\boldsymbol\beta^{ij}, \alpha^{ij}) \in C$, then $|\beta_i-\beta_j|\le\alpha^{ij}$, therefore \eqref{prob2} is indeed equivalent to \eqref{eq:tv}.
Moreover, because $C$ is a proper cone, we can rewrite \eqref{prob2} as follows,
\begin{equation}
\label{prob3}
\begin{split}
    & \min_{\forall (i,j) \in E,  \{\boldsymbol\beta^{ij} \in \mathbb{R}^d, \alpha^{ij} \in \mathbb{R}\}}{\frac{1}{2}\Vert \boldsymbol\beta-\mathbf{z} \Vert_2^2 +  \sum_{(i,j)\in E}{\theta_{ij} \alpha^{ij}}}
    \\
    & ~~ s.t. ~~ \begin{bmatrix}
    \boldsymbol\beta^{ij}\\ \alpha^{ij} \\
    \end{bmatrix} \succeq_C 0 ~ \text{and}  ~\beta^{ij}_k = \begin{cases} \beta_k & k = i, j \\0 & else \end{cases}.
\end{split}
\end{equation}
where $\boldsymbol\beta\succeq_C 0 $$\iff$$ \boldsymbol\beta \in C$ is defined as generalized inequality \cite{boyd2004convex}. We call \eqref{prob3} the primal problem (which equals to the original TV problem). Since the primal problem is both convex and satisfies Slater's condition, strong Lagrange duality holds (under generalized inequality). We define the Lagrange function as
\begin{equation}
\label{lagr}
\begin{split}
    & ~~~~~L(\boldsymbol\beta, \alpha, \boldsymbol\xi, \tau) \\
    & = \frac{1}{2}\Vert \boldsymbol\beta-\mathbf{z} \Vert_2^2 + \sum_{(i,j)}{\theta_{ij} \alpha^{ij}}- \sum_{ij}{\begin{bmatrix}
    \boldsymbol\xi^{ij}\\ \tau^{ij} \\
    \end{bmatrix}^T
    \begin{bmatrix}
    \boldsymbol\beta^{ij}\\ \alpha^{ij} \\
    \end{bmatrix}} \\
    & ~~s.t. ~~\boldsymbol\xi^{ij}\in \mathbb{R}^d:
    \begin{bmatrix}
    \boldsymbol\xi^{ij}\\ \tau^{ij} \\
    \end{bmatrix} \succeq_{C^*} 0 ~\text{and}~ \xi^{ij}_k = 0~\text{if}~ k\neq i,j,
\end{split}
\end{equation}
where $(\boldsymbol\xi^{ij},\tau^{ij})$ are the Lagrange multipliers and $C^*$ is the dual cone of $C$, defined as $C^* = \{\mathbf{v}~|~\mathbf{w}^T\mathbf{v}\geq0,~ \forall \mathbf{w}\in C\}$. To formulate the dual problem, we take the derivative of $L(\cdot)$ with respect to the primal variables $(\boldsymbol\beta, \alpha)$ and we have
\begin{equation}
\label{equv}
    \boldsymbol\beta - \mathbf{z} - \sum_{ij}{\boldsymbol\xi^{ij}} = 0 ~~ \text{and} ~~ \theta_{ij} - \tau^{ij} = 0.
\end{equation}
By applying \eqref{equv} to \eqref{lagr}, the dual problem is written as
\begin{equation}
\label{dual}
\begin{split}
    & \max_{\forall (i,j) \in E,  \{\boldsymbol\xi^{ij} \in \mathbb{R}^d\}} -\frac{1}{2} \Vert \mathbf{z} + \sum_{(i,j)\in E}{\boldsymbol\xi^{ij}} \Vert_2^2 + \frac{1}{2}\Vert \mathbf{z}\Vert_2^2 \\
    & ~~ s.t.~~ (\boldsymbol\xi^{ij}, \tau^{ij}) \in C^*;~\xi^{ij}_k = 0, k\neq i,j.
\end{split}
\end{equation}
\begin{proposition}
\label{lem:dualcone}
    $(\boldsymbol\xi^{ij}, \tau^{ij}) \in C^*,~\xi^{ij}_k = 0, k\neq i,j \iff \xi^{ij}_i+\xi^{ij}_j=0,~|\xi^{ij}_i|\le \theta_{ij}$.
\end{proposition}
Please find the proof in the supplementary file. 

Accordingly, the dual problem becomes
\begin{equation}
\label{dual2}
\begin{split}
     & \min_{\forall (i,j) \in E,  \{\boldsymbol\xi^{ij} \in \mathbb{R}^d\}} \frac{1}{2} \Vert \mathbf{z} - \sum_{(i,j)\in E}{\boldsymbol\xi^{ij}} \Vert_2^2 \\
     &~~ s.t.~~\xi^{ij}_k = 0, k\neq i,j, ~ \xi^{ij}_i+\xi^{ij}_j=0,~|\xi^{ij}_i|\le \theta_{ij},
\end{split}
\end{equation}
where we omit $\Vert \mathbf{z}\Vert_2^2$ from the objective since it is a constant with respect to $\boldsymbol\xi$s and we also changed the sign of all $\boldsymbol\xi$s for better illustration of the flows.
\begin{figure}
    \centering
    \includegraphics[width = 0.9\columnwidth]{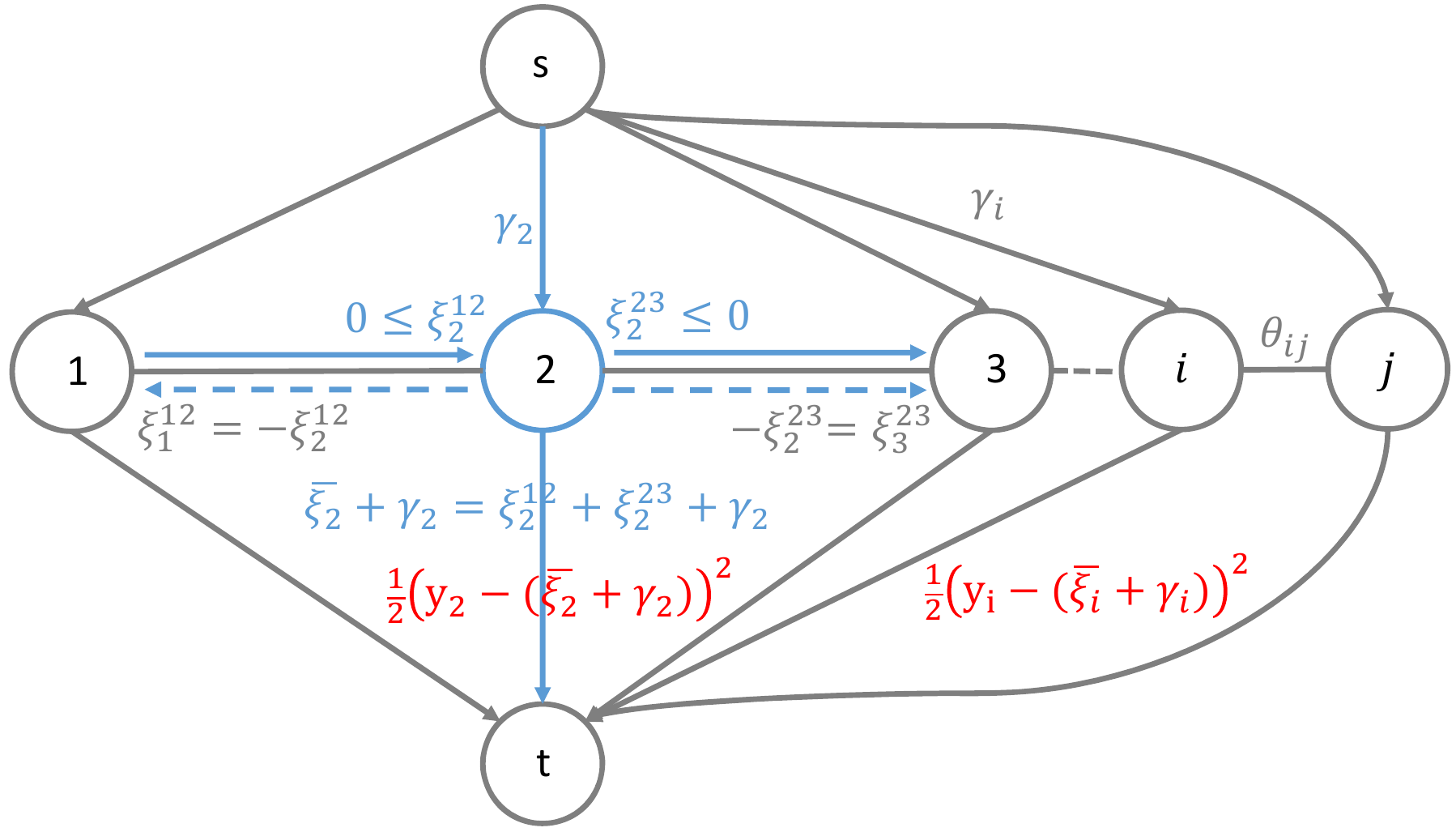}
    \caption{\small{Graph structure and the flows. Each source-to-node edge (s,i) has 0 cost and a maximum capacity $\gamma_i$ and a minimum capacity $\gamma_i$, this ensures a flow $\gamma_i$; each node-to-node edge (i,j) has 0 cost and a maximum capacity $\theta_{ij}$; each node-to-sink edge (i,t) has infinite capacity and a cost (in red) $\frac{1}{2}(y_i-(\bar \xi_i+\gamma_i))$, where $\bar \xi_i = \sum_{j,(i,j)\in E}{\xi^{ij}_i}$. All flows through node 2 is highlighted in blue for an example.}}
\label{fig:flow}
\end{figure}

Problem \eqref{dual2} can be viewed as the following minimum quadratic cost flow formulation,
\begin{equation}
\label{dual4}
\begin{split}
    & \min_{\forall (i,j) \in E,  \{\boldsymbol\xi^{ij} \in \mathbb{R}^d\}} \frac{1}{2} \Vert \mathbf{y} - (\sum_{(i,j)\in E}{\boldsymbol\xi^{ij}}+\boldsymbol\gamma) \Vert_2^2 \\
    & ~~ s.t.~~\xi^{ij}_k = 0, k\neq i,j, ~ \xi^{ij}_i+\xi^{ij}_j=0,~|\xi^{ij}_i|\le  \theta_{ij},
\end{split}
\end{equation}
where we have induced $\boldsymbol\gamma$ ($\gamma_i = \max{\{|z_i|,\sum_{j,(i,j)\in E}{ \theta_{ij}}\}}$) and denote $\mathbf{y}$$ =$$ \mathbf{z} + \boldsymbol\gamma$ to ensure $\mathbf{y}\geq \mathbf{0}$ and $(\boldsymbol\xi^{ij}+\boldsymbol\gamma)\geq \mathbf{0}$.
Thus each feasible $\boldsymbol\xi$ of  \eqref{dual4} is a possible flow on graph $G$$=$$(V,E)$. Since $\xi^{ij}_i+\xi^{ij}_j=0$, $|\xi^{ij}|$ can denote a flow on edge (i,j) such that $\xi^{ij}_i\geq 0$ denotes a flow coming into node $i$ and $\xi^{ij}_i\le 0$ denotes a flow leaving node $i$. Figure \ref{fig:flow} illustrates such flows by taking node 2 as an example. Thus to minimize the objective of  \eqref{dual4} is equivalent to computing a minimum cost flow on this graph. Since the cost is quadratic with respect to the flow, this problem is a minimum quadratic cost flow problem. According to \cite{hochbaum1995strongly,mairal2011convex}, this type of problems can be efficiently solved via fast flow algorithms including but not limited to the parametric flow \cite{GGT89}. Note that in \cite{bo2014gfused}, TV is shown equivalent to a minimum norm point (MNP) problem under submodular constraints which is solved via parametric flow. We now discuss the relation between the dual problem i.e. \eqref{dual4} or \eqref{dual2} and the MNP considered in \cite{bo2014gfused}.

\begin{table}[t]
\caption{Runtime (in sec.) comparison of the proposed algorithm with CVX. $d$ is the data dimensionality. $\ast$ indicates the test did not finish within 24 hours.}
\label{tab:eff}
\begin{center}
\begin{tabular}{cccccc}
\hline
$d$  & $400$ & $900$ & $2500$ & $4900$  & $10000$\\
\hline
CVX      & 5.22 & 33.71 & 889.80 & 1.08$e^{4}$  &   $ \ast$   \\
Ours      & 0.87 & 2.39  &  15.95 & 64.33 &   321.93  \\
\hline
\end{tabular}
\end{center}
\end{table}

Recall that the MNP problem is defined as follows
\begin{equation}
\label{eq:mnp}
    \min_{\mathbf{s}\in \mathbb{R}^d, \mathbf{s}\in B(\lambda f_c)} {\Vert  \mathbf{z} - \mathbf{s}  \Vert_2^2},
\end{equation}
where $fc(S)$ is a cut function, defined as $f_c(\mathcal{S}) = \sum_{i\in \mathcal{S},j \in \mathcal{V} \backslash \mathcal{S}}{w_{ij}}$ and $B(\cdot)$ is the base polyhedron of $f_c$.
\begin{proposition}
\label{prop1}
For any minimizer $\boldsymbol\xi^*$ of  \eqref{dual2}, define $\hat{\mathbf{s}}$ such that $\hat{\mathbf{s}} = \sum_{(i,j)\in E}{{\boldsymbol\xi^*}^{ij}}$, then $\hat{\mathbf{s}}$ is a minimizer of  \eqref{eq:mnp}. For any minimizer $\mathbf{s}^*$ of  \eqref{eq:mnp}, there exists a decomposition such that $\mathbf{s}^*$$=$$\sum_{(i,j)\in E}{\hat{\boldsymbol\xi^{ij}}}$, where $\hat{\boldsymbol\xi}$ is one minimizer of  \eqref{dual2}.
\end{proposition}

According to Prop. \ref{prop1},  the MNP problem i.e. \eqref{eq:mnp} can be viewed as a special case of  \eqref{dual2} (the conic dual), where $\sum_{(i,j)\in E}{{\boldsymbol\xi}^{ij}}$$=$$\mathbf{s}$. Moreover, since  \eqref{dual2} has relatively ``looser" constraints, it is possible to devise more efficient algorithms (than parametric flow) to solve  \eqref{dual2} and thereafter TV. For example, in \cite{mairal2011convex}, a faster (than parametric flow) flow algorithm is proposed to solve their specific minimum quadratic flow problem. Hence, the conic dual perspective opens a new opportunity to solve the famous TV problem more efficiently.

{\bf Optimization summary.} In summary, by applying Prop. \ref{prop:cfista}, we can solve $n^2$GFL by iteratively solving \eqref{eq:po}. By applying Prop. \ref{prop:totv}, we further reduce \eqref{eq:po} to the TV problem defined in \eqref{eq:tv}, we then transform it to a minimum cost flow algorithm via conic duality and solve it by a fast flow algorithm.

In Tab. \ref{tab:eff}, we compare the proposed algorithm with an off-the-shelf solver on synthetic data. We generate a random $\boldsymbol\beta \in \mathbb{R}^d$ and a 2D grid graph of $d$ nodes with each node having four neighbors. We then generate $N=d/2$ samples: $\boldsymbol{x}_i\in \mathbb{R}^d$ and $y_i=\boldsymbol{\beta}^T\boldsymbol{x}_i+0.01 n_i$, where $\boldsymbol{x}_i$ and $n_i$ are drawn from the standard normal distribution. All experiments are carried out on an Intel(R) Core(TM) i7-3770 CPU at 3.40GHz. The experiments show that the proposed optimization algorithm is more efficient and scalable.

\begin{figure*}
	\centering
    \subfigure{
        \includegraphics[width=.35\columnwidth]{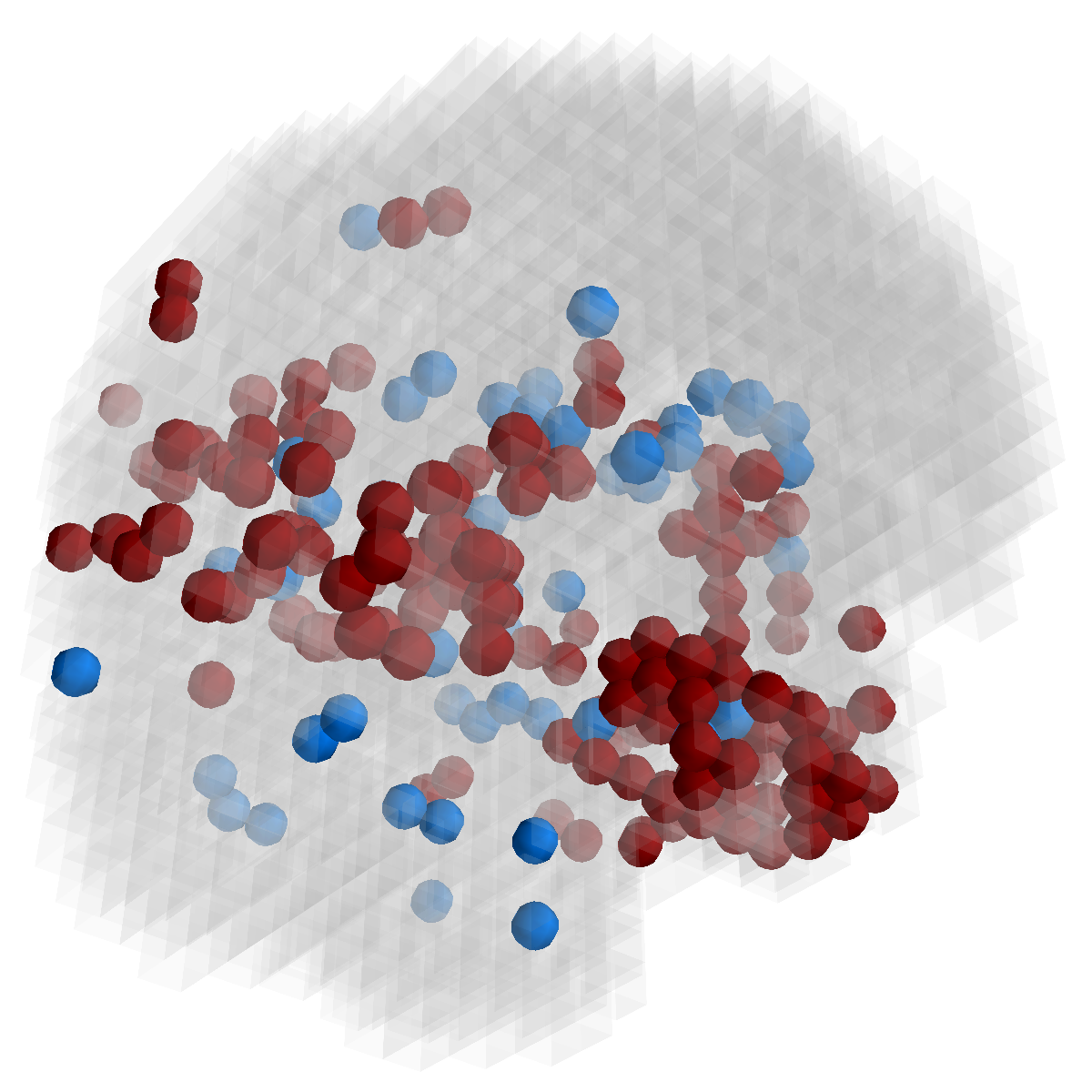}
        }
    \subfigure{
        \includegraphics[width=.35\columnwidth]{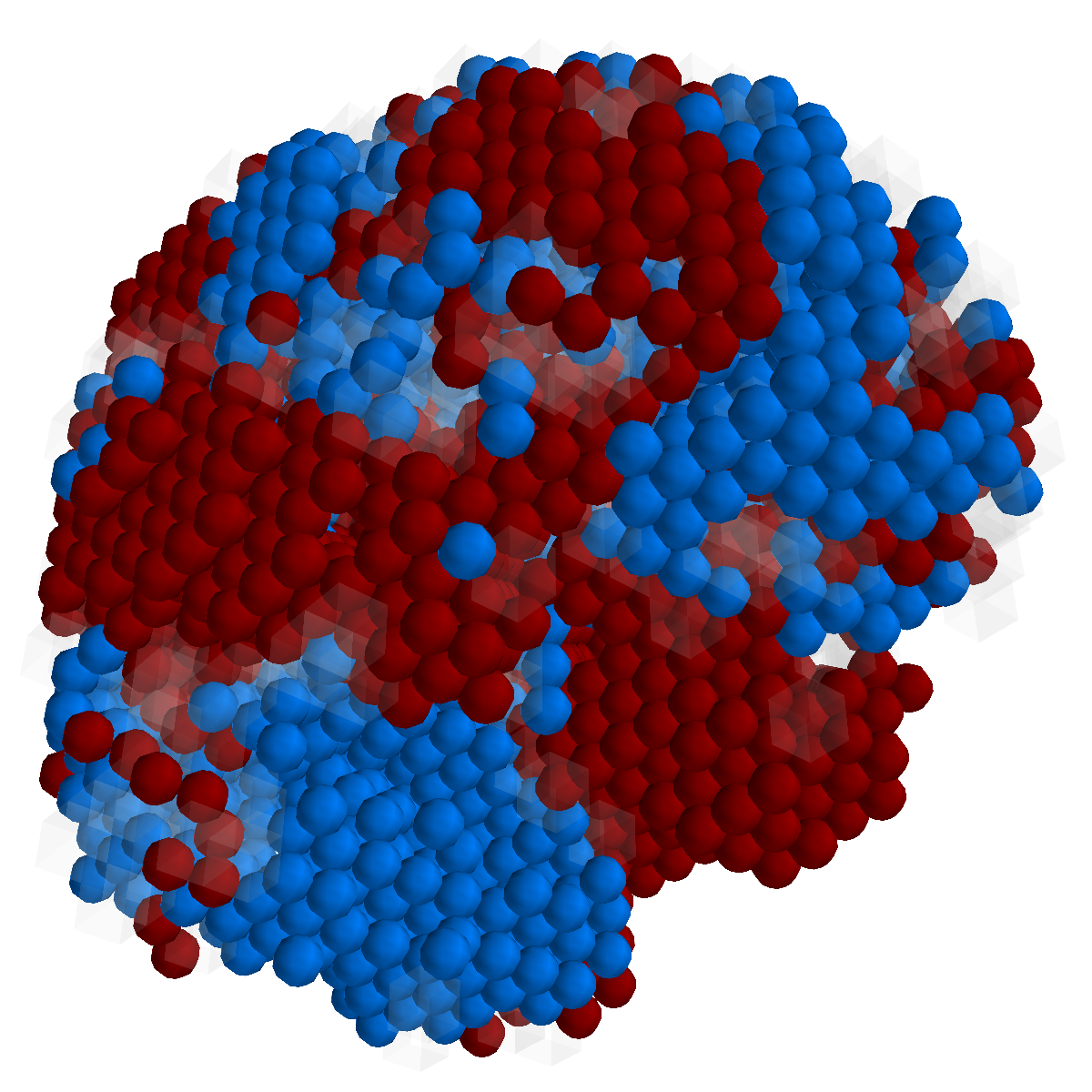}
        }
    \subfigure{
        \includegraphics[width=.35\columnwidth]{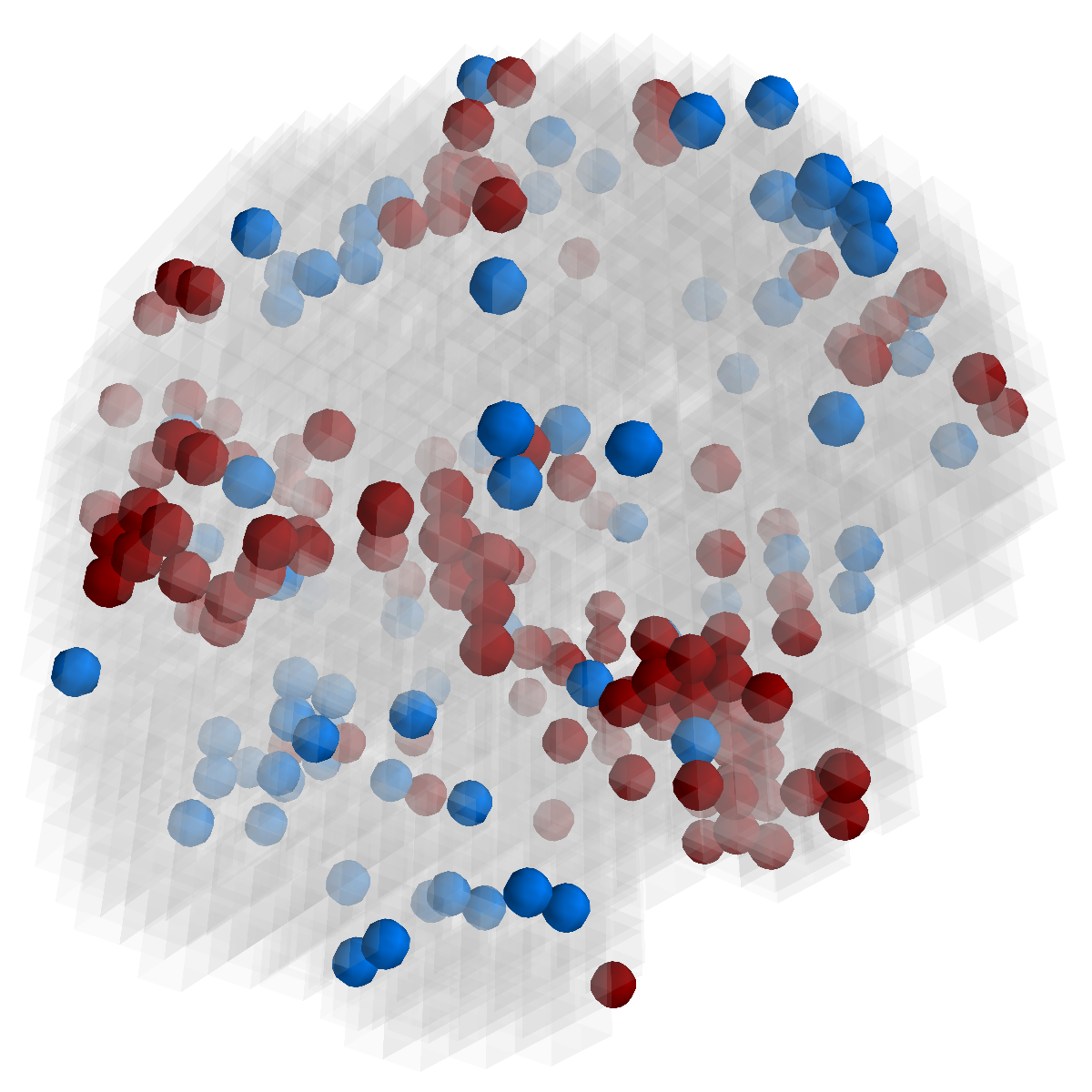}
        }
    \subfigure{
        \includegraphics[width=.35\columnwidth]{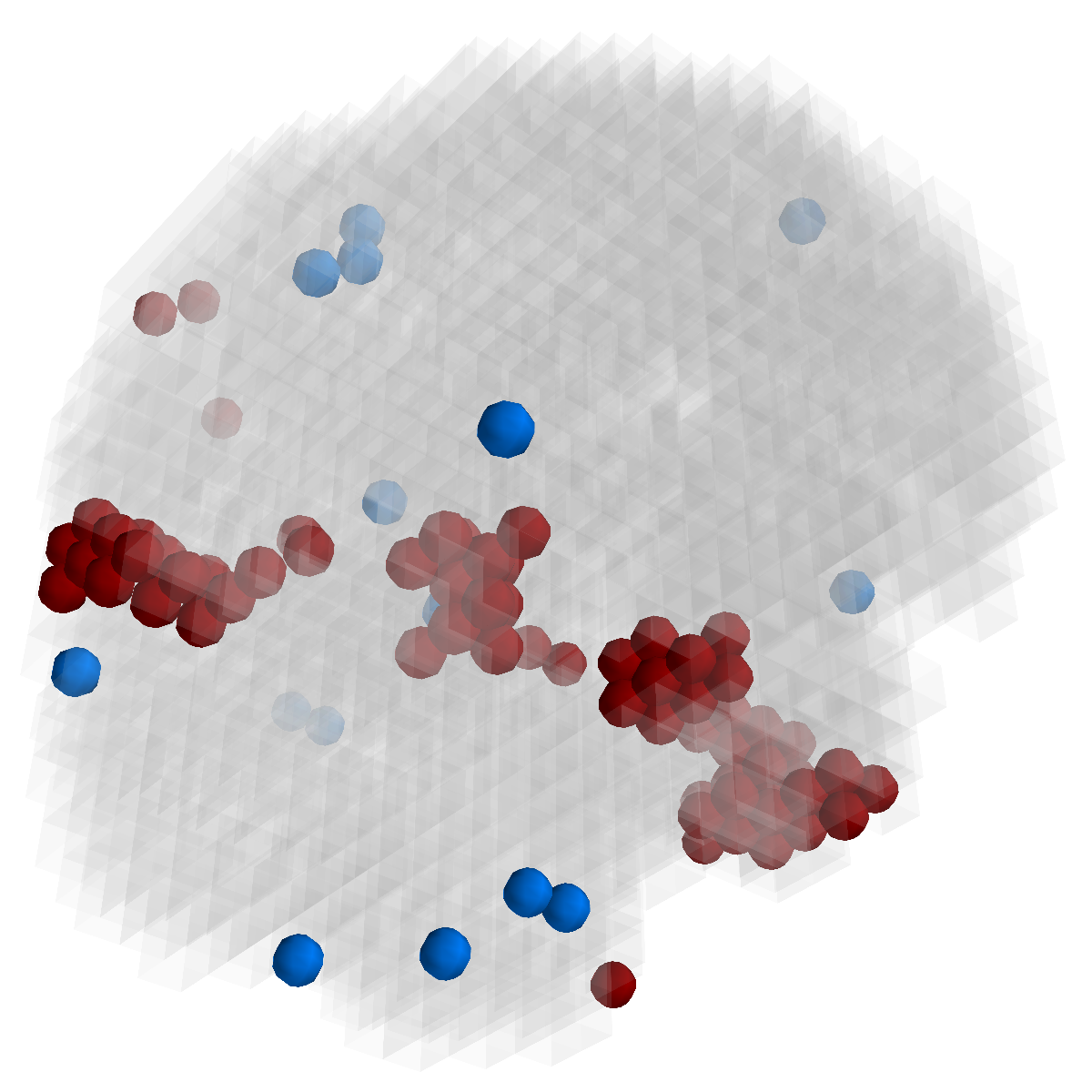}
        }
    \subfigure{
        \includegraphics[width=.35\columnwidth]{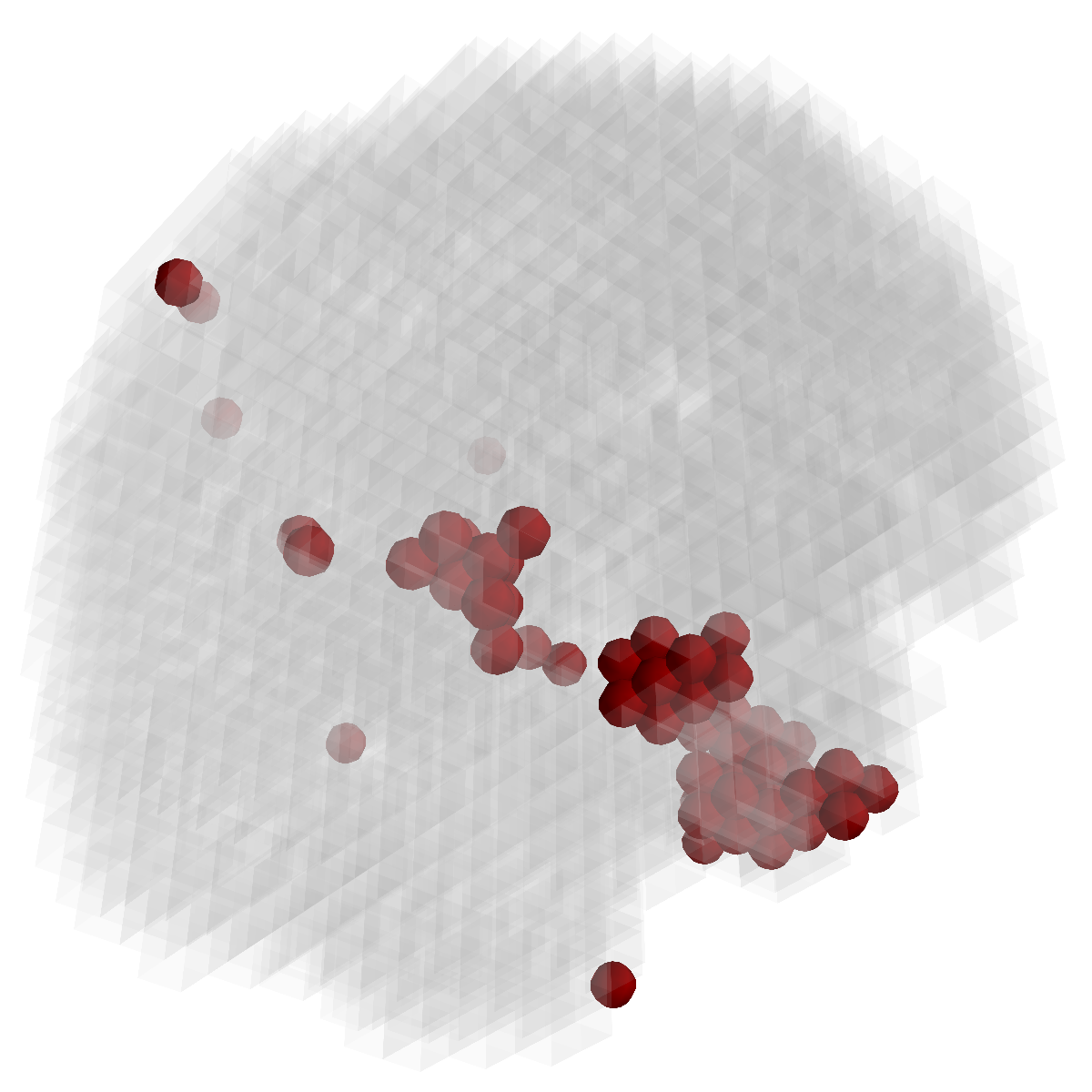}
        }
    \subfigure{
        \includegraphics[width=.35\columnwidth]{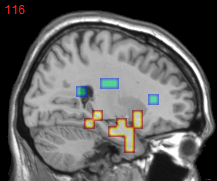}
        }
    \subfigure{
        \includegraphics[width=.35\columnwidth]{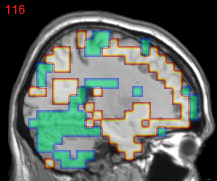}
        }
    \subfigure{
        \includegraphics[width=.35\columnwidth]{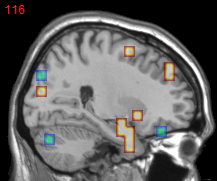}
        }
    \subfigure{
        \includegraphics[width=.35\columnwidth]{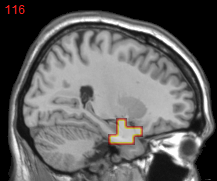}
        }
    \subfigure{
        \includegraphics[width=.35\columnwidth]{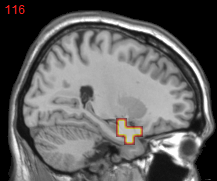}
        }
    \setcounter{subfigure}{0}
    \subfigure[T-test(LDA)]{
        \label{fig:15ad3}
        \includegraphics[width=.35\columnwidth]{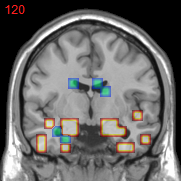}
        }
    \subfigure[LapL]{
        \label{fig:15ad3}
        \includegraphics[width=.35\columnwidth]{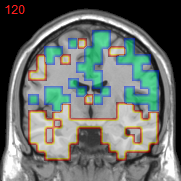}
        }
    \subfigure[lasso]{
        \label{fig:15ad1}
        \includegraphics[width=.35\columnwidth]{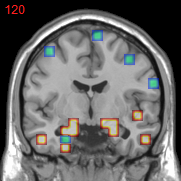}
        }
    \subfigure[GFL]{
        \label{fig:15ad2}
        \includegraphics[width=.35\columnwidth]{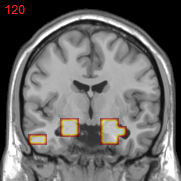}
        }
    \subfigure[$n^2$GFL]{
        \label{fig:15ad3}
        \includegraphics[width=.35\columnwidth]{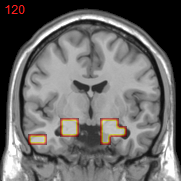}
        }
	\caption{\small{Feature selection by different models. The top row illustrates selected voxels in a 3D
                    model (voxels with positive $\beta$ are in brown and negative ones are in blue), the mid and bottom rows illustrate the corresponding projections on brain slices.}}
\label{fig_GSR1}
\end{figure*}

\section{Application to the Diagnosis of AD}
\label{ssec:ad}

In the diagnosis of AD, two fundamental issues are AD/NC (Normal Control) classification and MCI/NC (Mild Cognitive Impairment) classification. Let $\mathbf{x}_i\in \mathbb{R}^d$ be subjects' sMRI voxels and $y_i=\{-1,1\}$ be the disease status (AD/NC or MCI/NC). Since the problems are classifications, we use the logistic regression as the loss term
\begin{equation}
\label{eq:lg}
    l(\boldsymbol\beta) = \sum_{i=1}^{N}{ \log{(1+ \exp{(-y_i(\boldsymbol\beta^T \mathbf{x}_i+c))}})},
\end{equation}
where $c\in \mathbb{R}$ is the bias parameter (to be learned). For the graph structure, we define each voxel as a node and their spatial adjacency as the edges, i.e. $w_{ij}$$=$$1$ if voxels $i$ and $j$ are adjacent and $0$ otherwise.
The data are obtained from the Alzheimer's Disease Neuroimaging Initiative (ADNI) database\footnote{http://adni.loni.ucla.edu}. We split all the baseline data into 1.5T and 3.0T MRI scans datasets (named 15T and 30T).
64 AD patients, 90 NC and 208 MCI patients are included in our 15T dataset; 66 AD patients and 110 NC are included in our 30T dataset. (Most 30T MCI data are in an on-going phase and are not included).
Data preprocessing follows the DARTEL VBM pipeline \cite{ashburner2007fast} as commonly done in the literature. 2,527 8$\times$8$\times$8 mm$^3$ size voxels that have values greater than 0.2 in the mean gray matter population template serve as the input features. We design experiments on three  tasks, namely, 15ADNC, 30ADNC, 15MCINC.

\begin{table}[t] 
\small
\caption{Classification accuracies.}
\label{tab:accr}
\begin{center}
\begin{tabular}{m{32pt}m{12pt}m{12pt}m{20pt}m{12pt}m{12pt}m{12pt}m{20pt}}
\toprule
 & SVM & LR & MLDA & LapL & lasso & GFL & $n^2$GFL \\
\midrule
15ADNC  & 83.1 & 83.1 & 83.1 & 84.4 & 85.7 & 85.1 & \textbf{86.4} \\
30ADNC  & 87.5 & 86.9 & 83.5 & 87.5 & 88.6 & 85.8 & \textbf{90.3} \\
15MCINC & 71.1 & 70.5 & 59.4 & 70.1 & 70.8 & 69.8 & \textbf{71.8} \\
\bottomrule
\end{tabular}
\end{center}
\end{table}

\begin{figure*}
	\centering
    \subfigure{
         \includegraphics[width=.33\columnwidth]{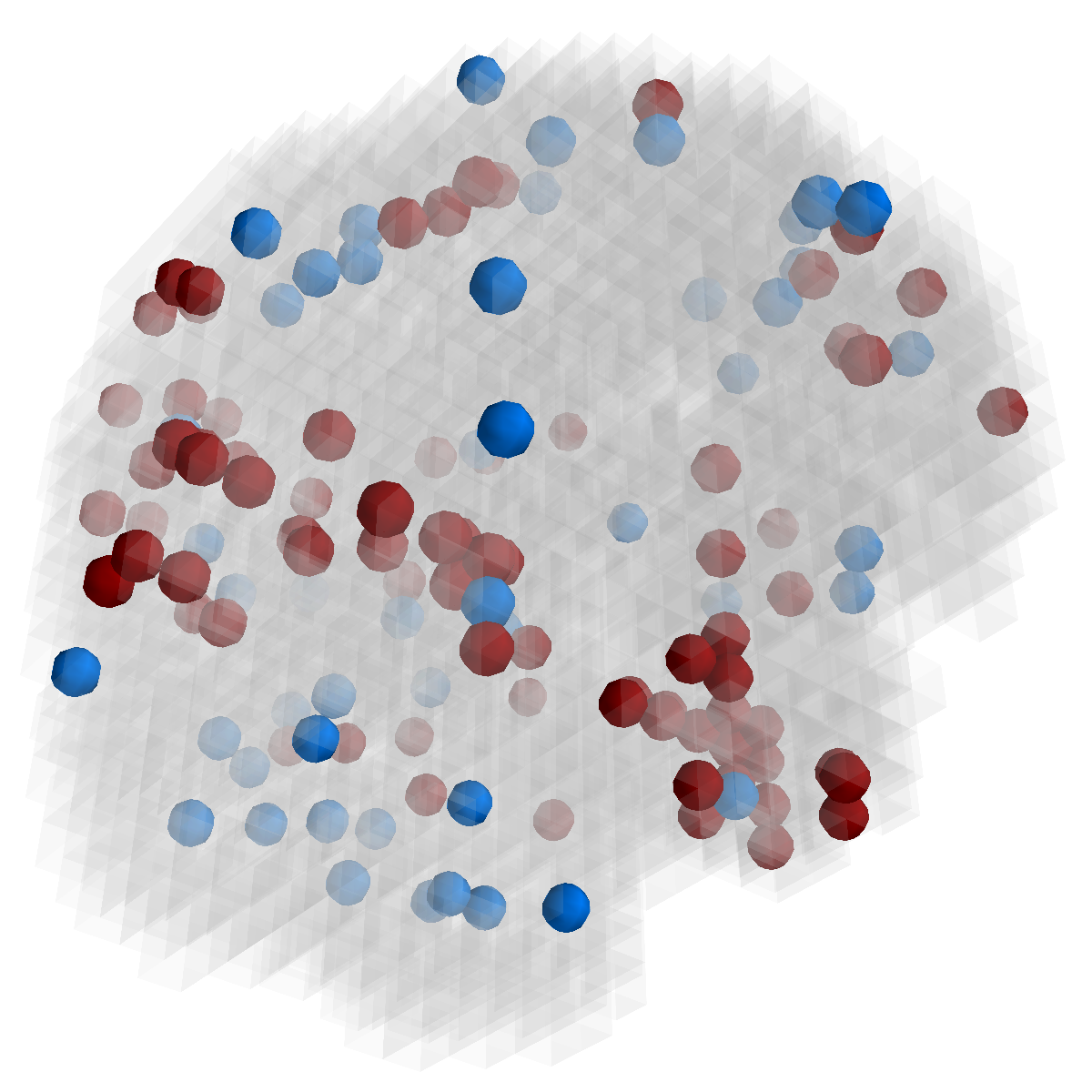}
        }
        \hspace{-7px}
    \subfigure{
         \includegraphics[width=.33\columnwidth]{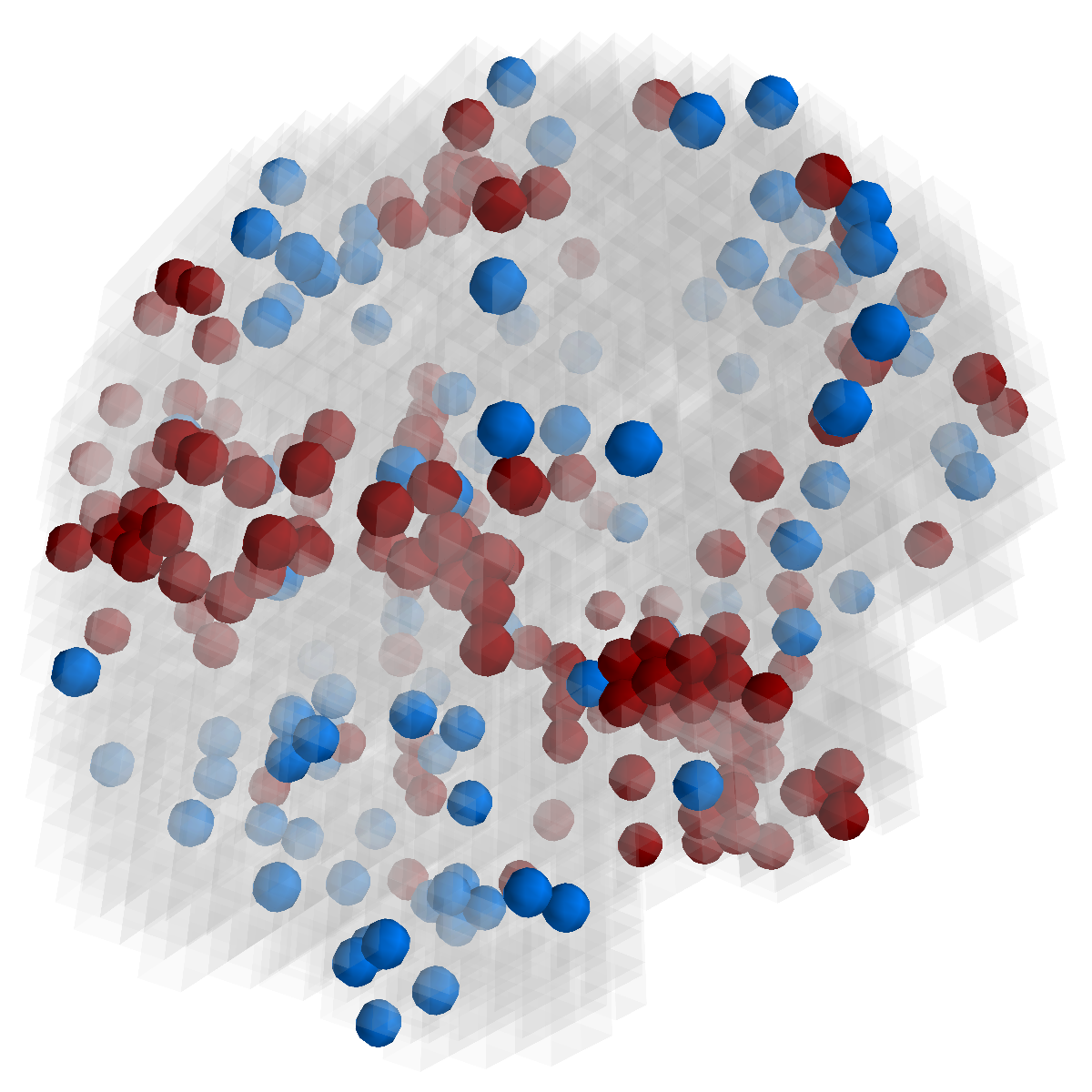}
        }
        \hspace{-7px}
    \subfigure{
         \includegraphics[width=.33\columnwidth]{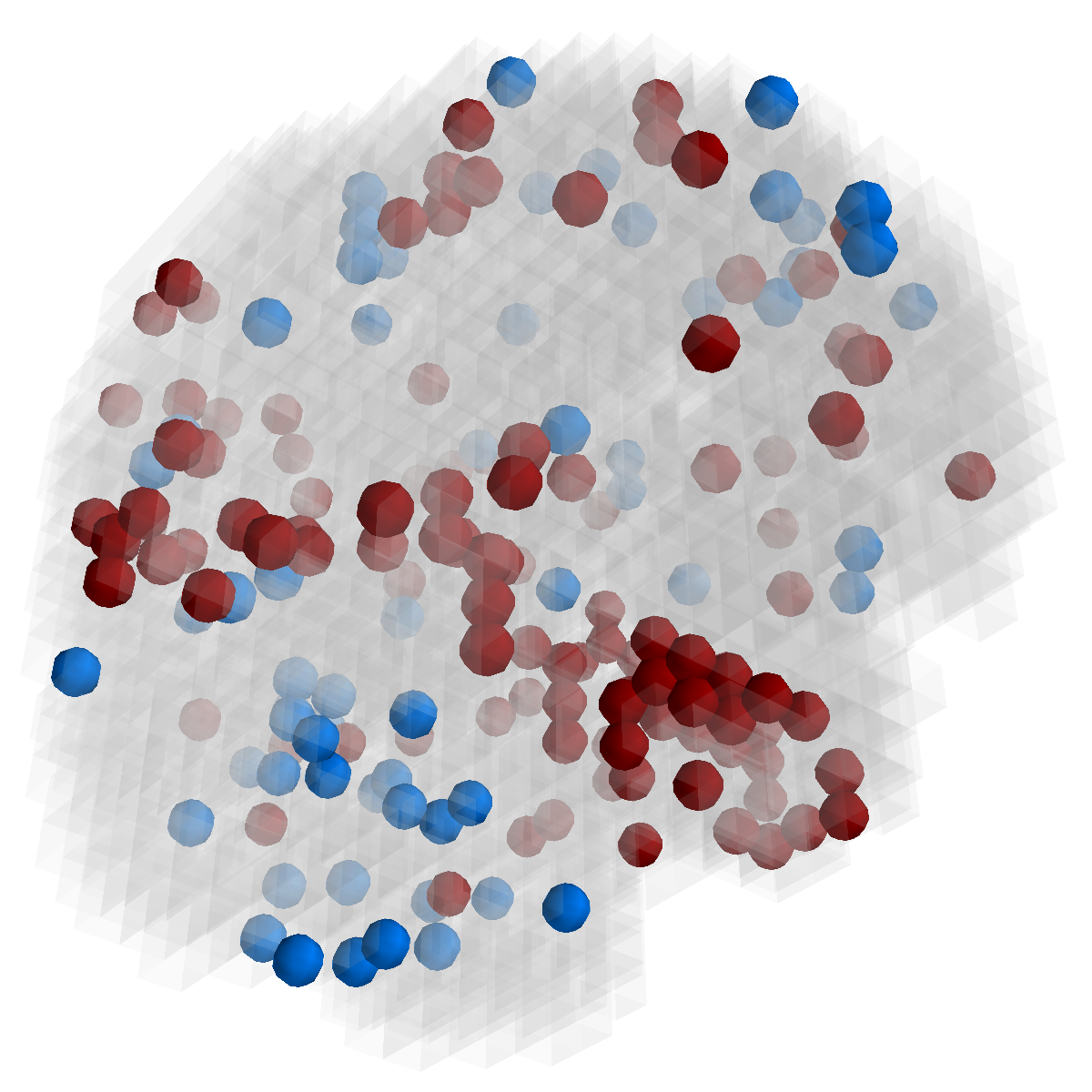}
        }
        \hspace{-7px}
    \subfigure{
         \includegraphics[width=.33\columnwidth]{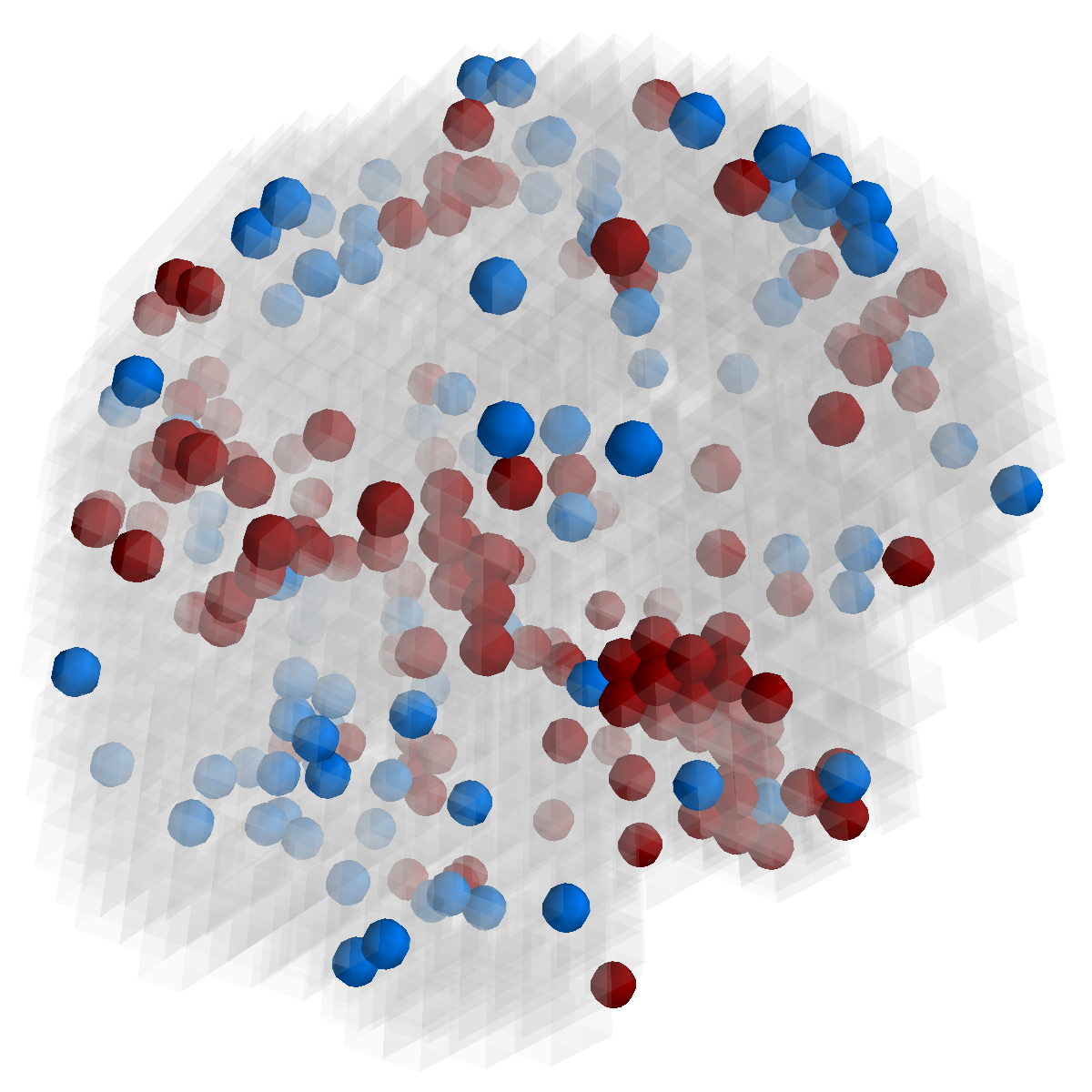}
        }
        \hspace{-7px}
    \subfigure{
         \includegraphics[width=.33\columnwidth]{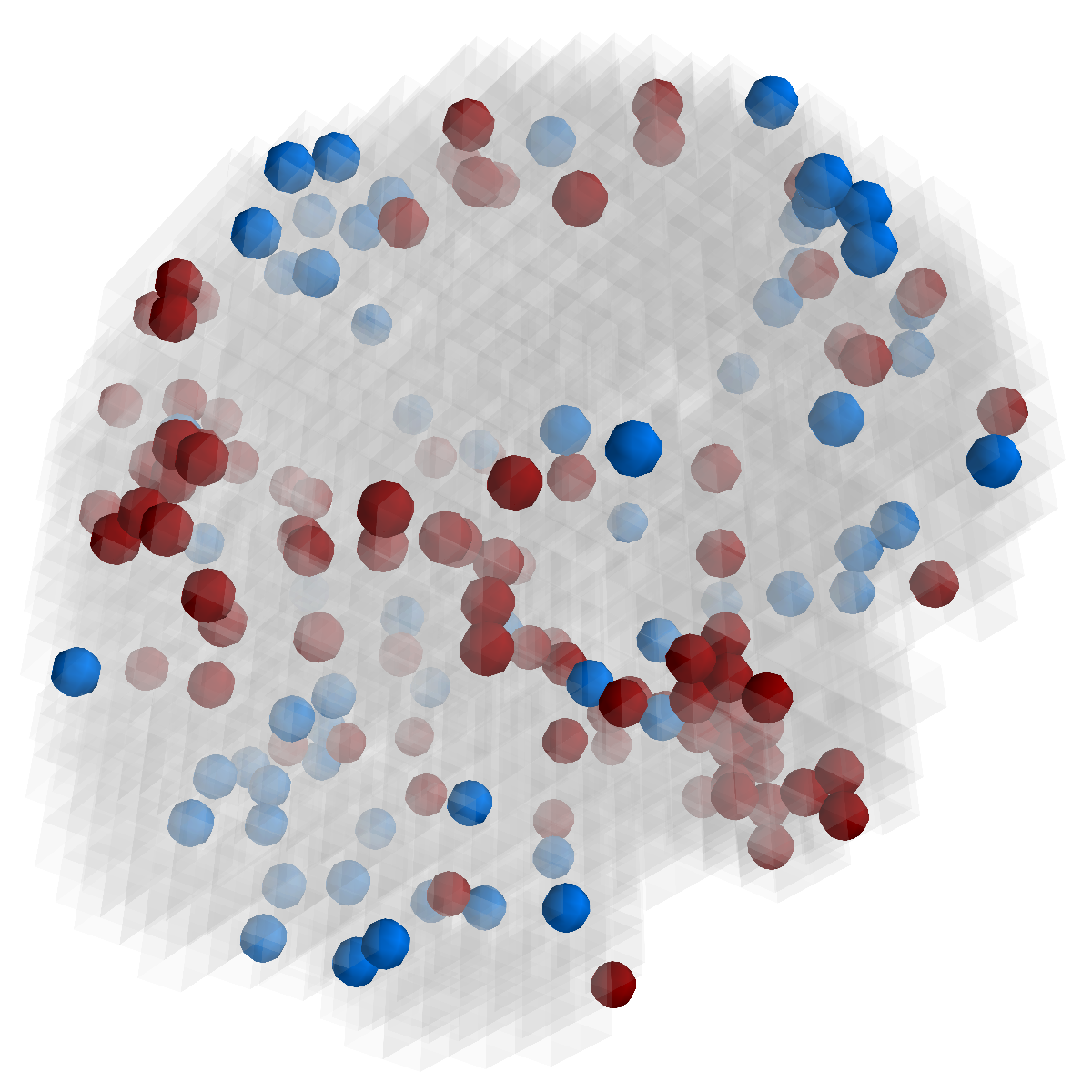}
        }
        \hspace{-7px}
    \subfigure{
         \includegraphics[width=.33\columnwidth]{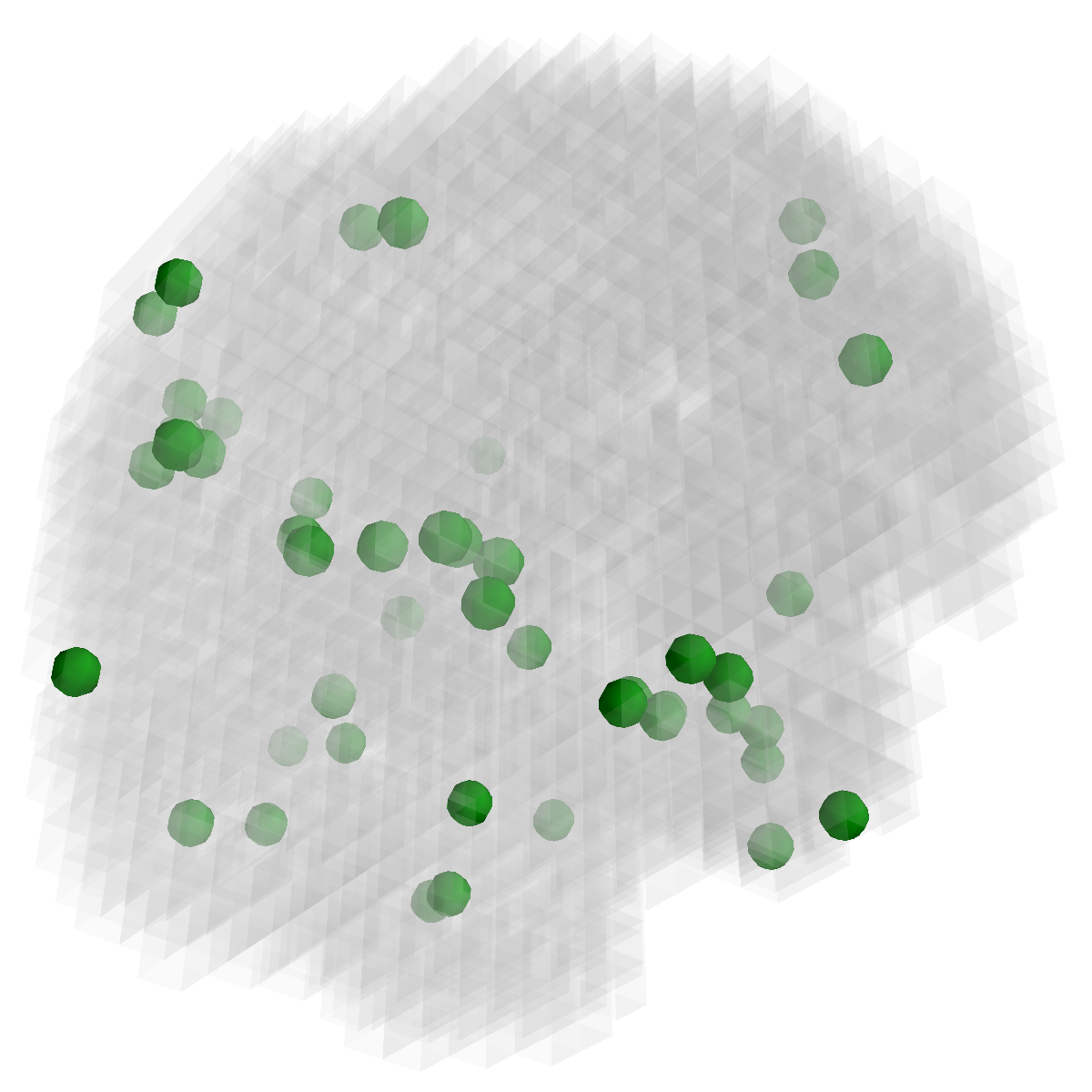}
        }
    \vspace{-3.5mm}

    \subfigure{
         \includegraphics[width=.33\columnwidth]{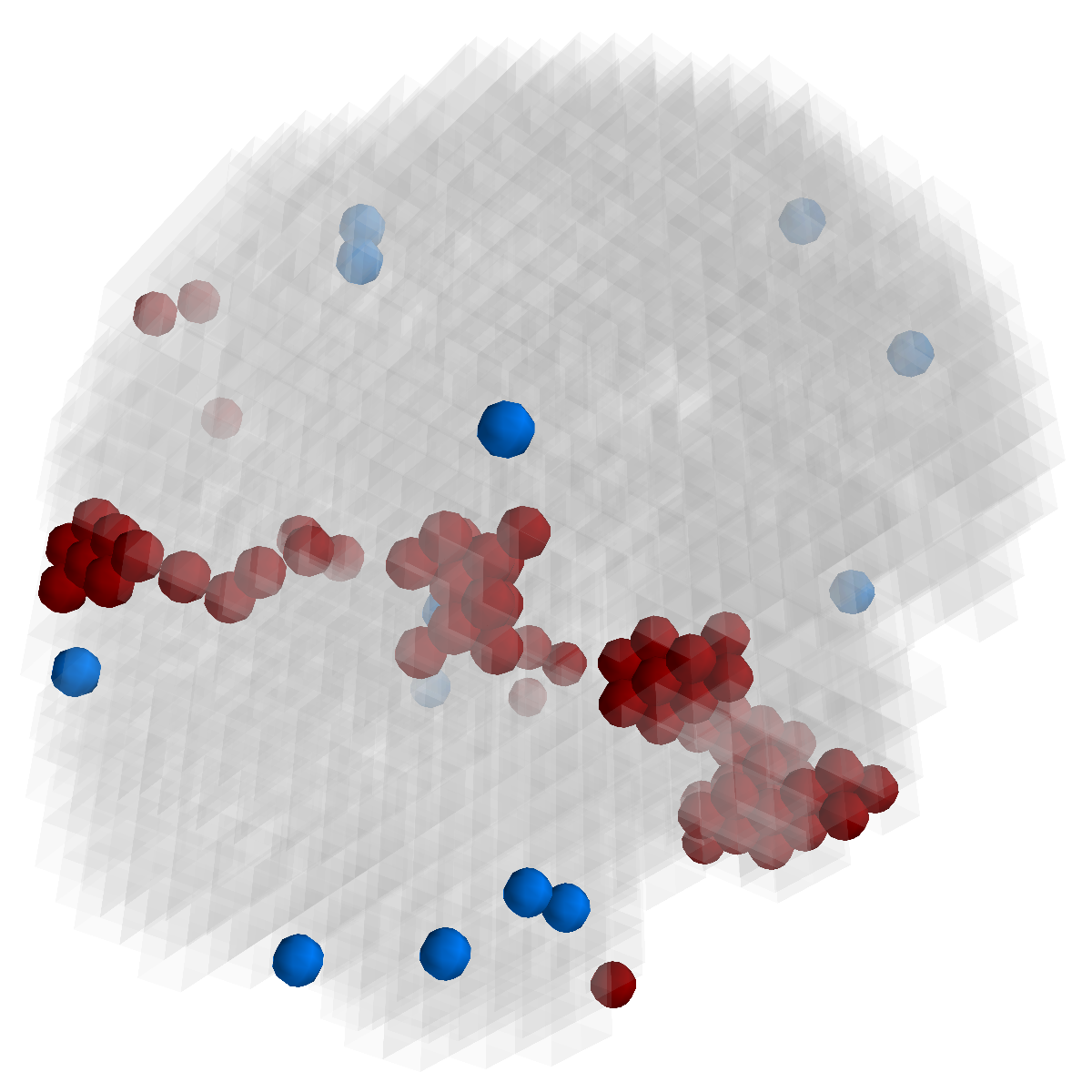}
        }
        \hspace{-7px}
    \subfigure{
         \includegraphics[width=.33\columnwidth]{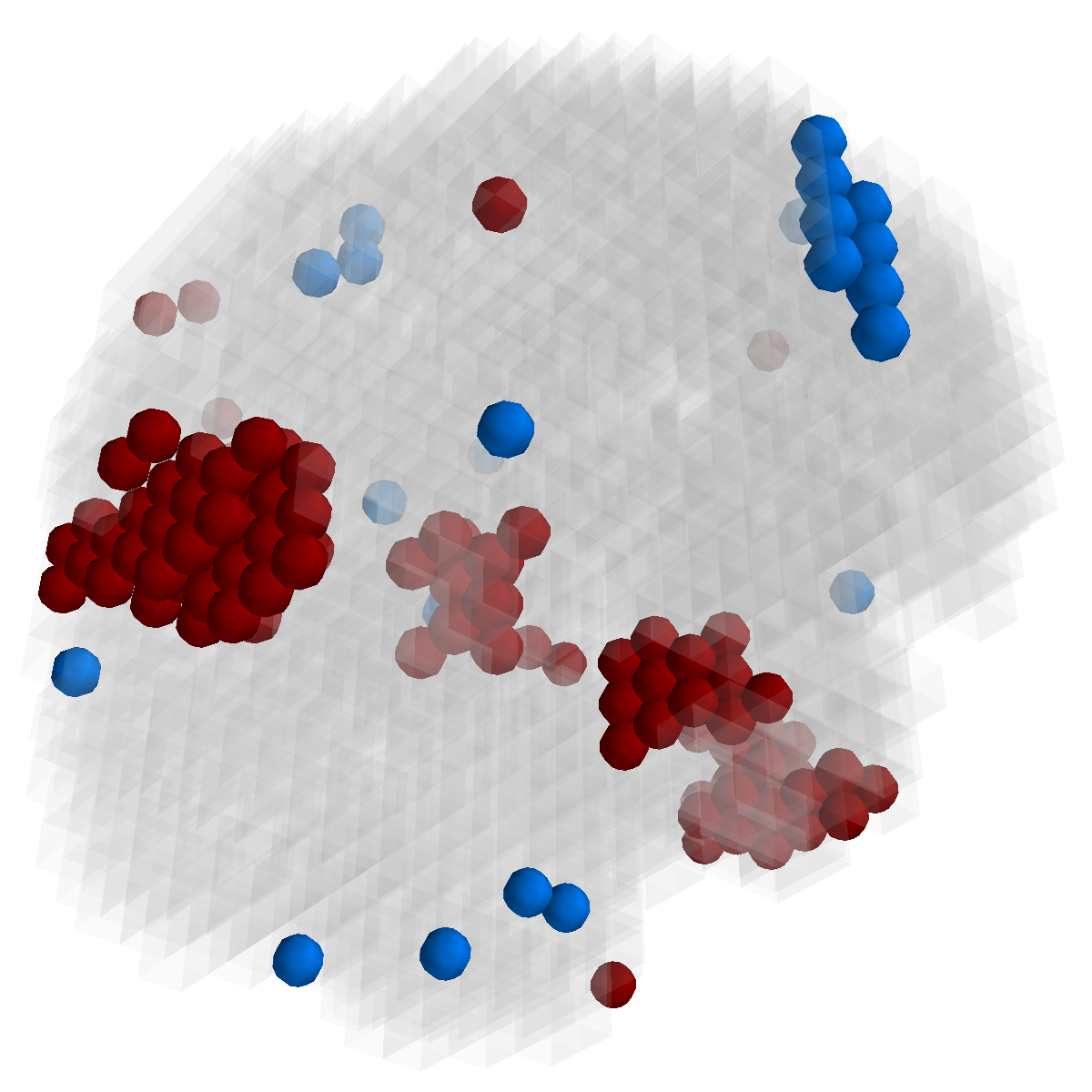}
        }
        \hspace{-7px}
    \subfigure{
         \includegraphics[width=.33\columnwidth]{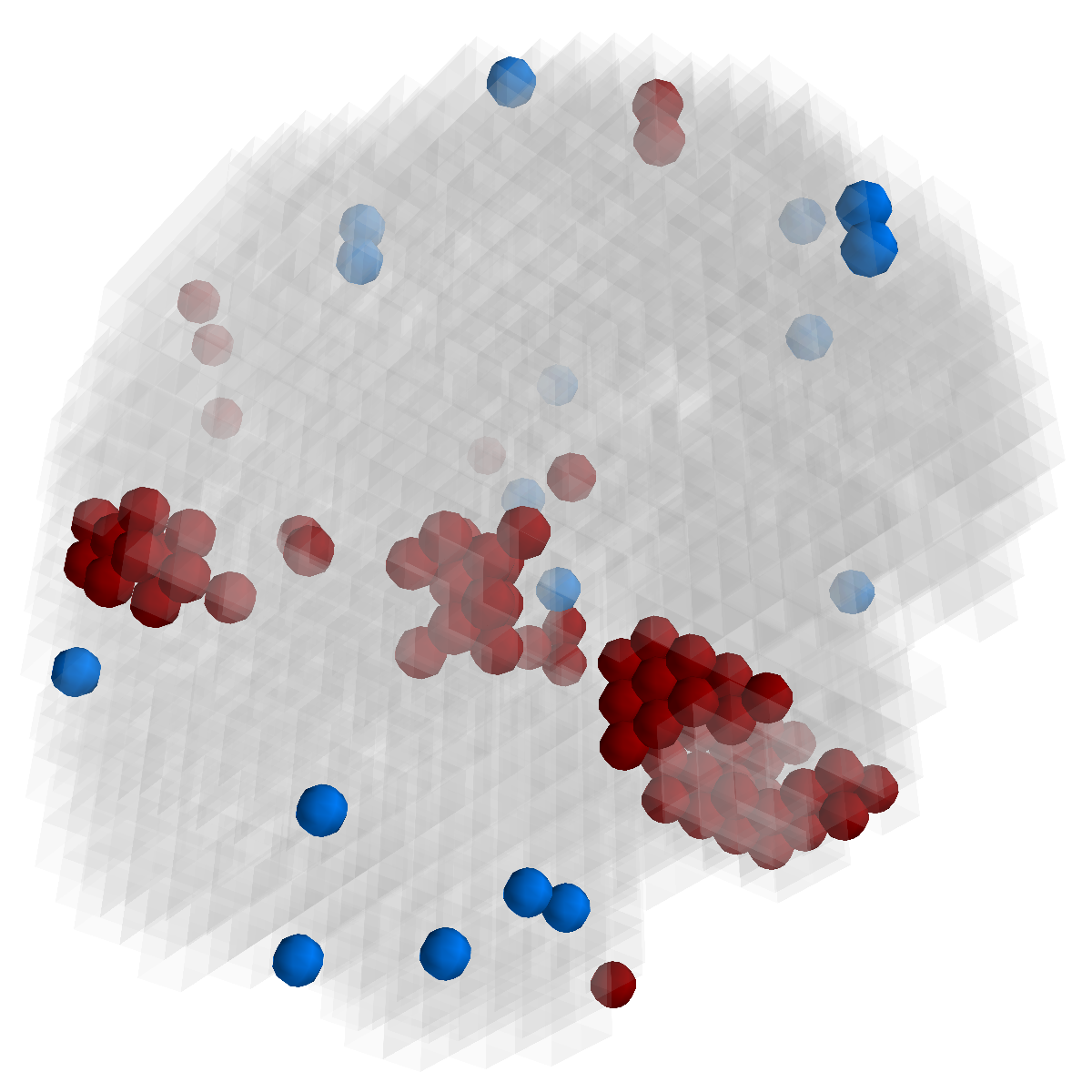}
        }
        \hspace{-7px}
    \subfigure{
         \includegraphics[width=.33\columnwidth]{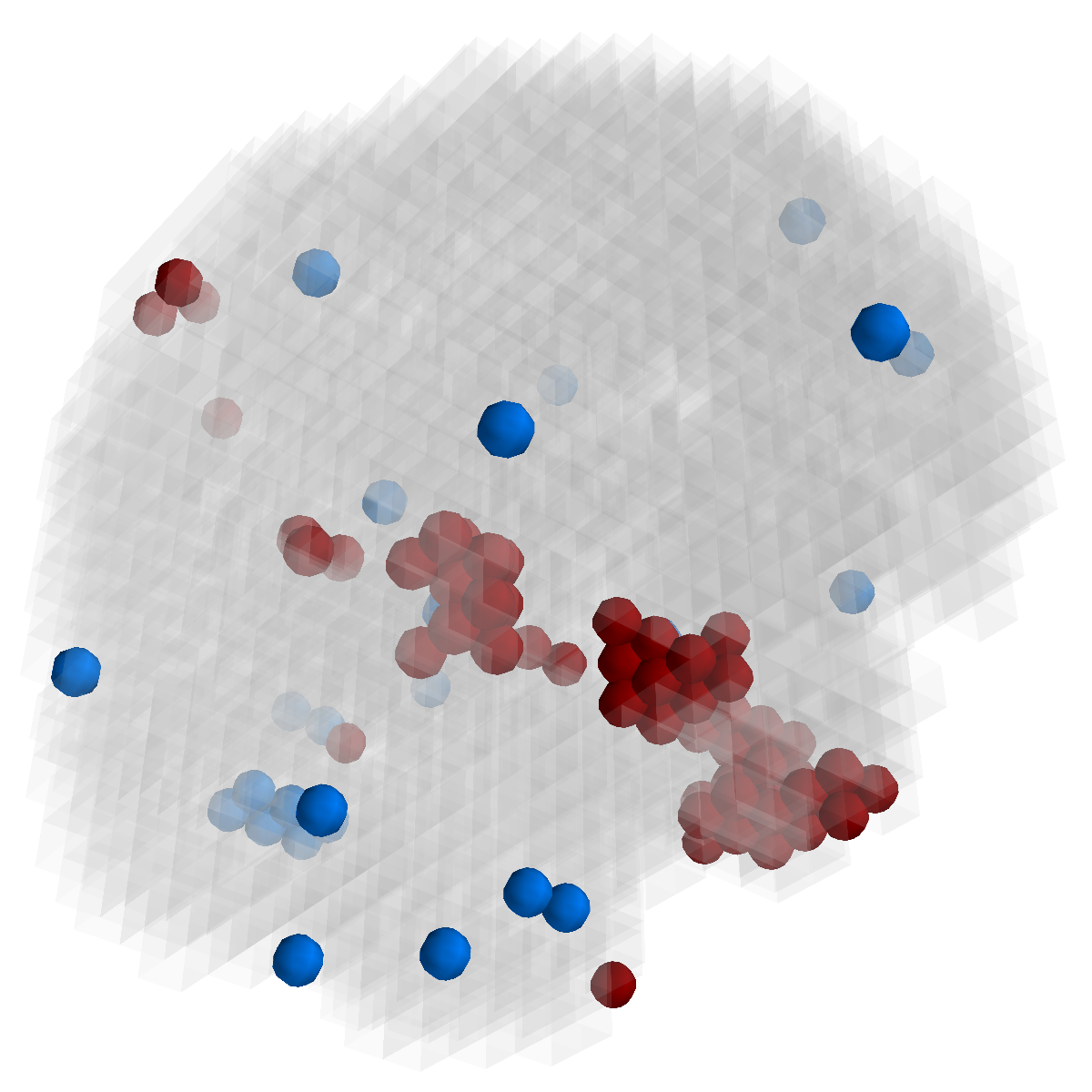}
        }
        \hspace{-7px}
    \subfigure{
         \includegraphics[width=.33\columnwidth]{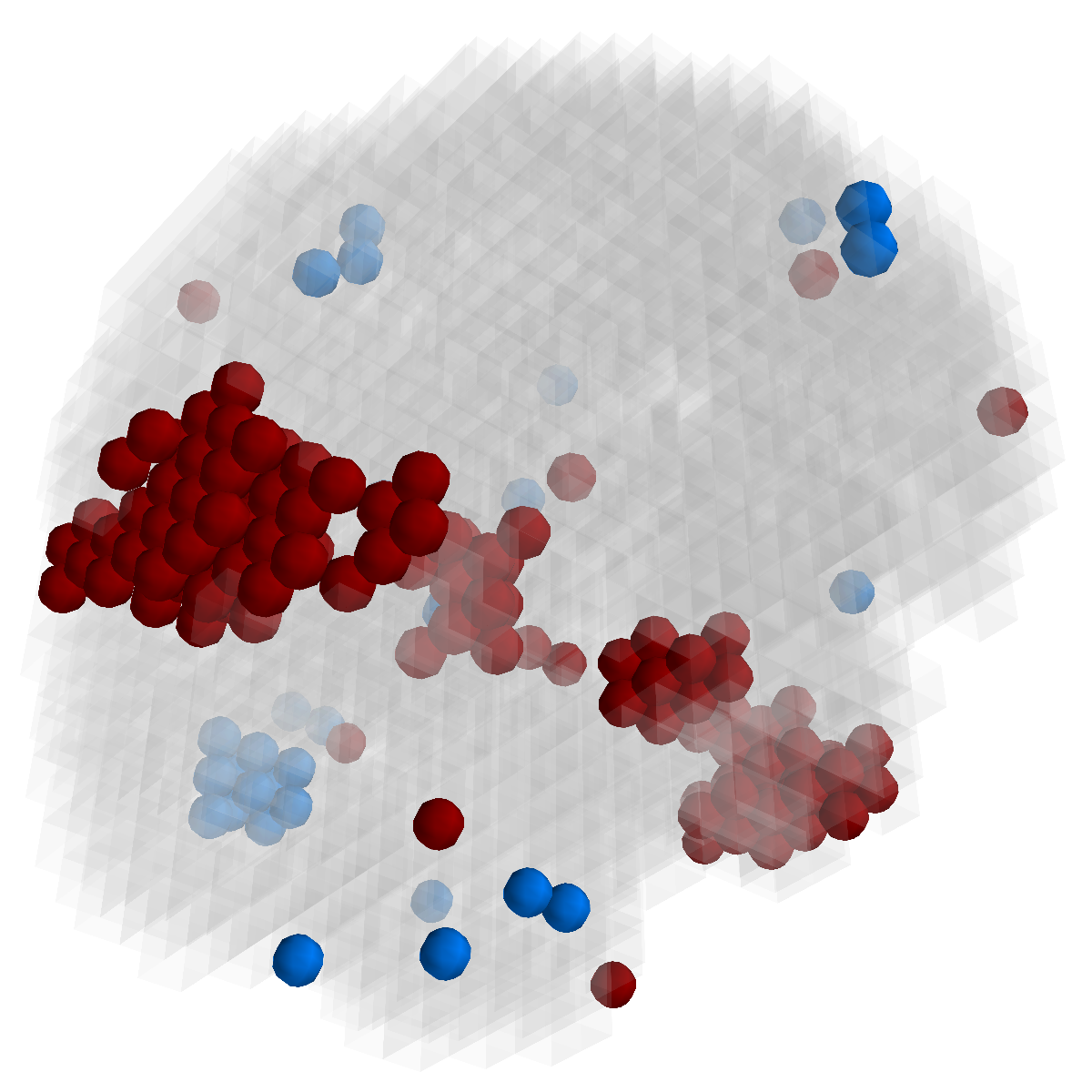}
        }
        \hspace{-7px}
    \subfigure{
         \includegraphics[width=.33\columnwidth]{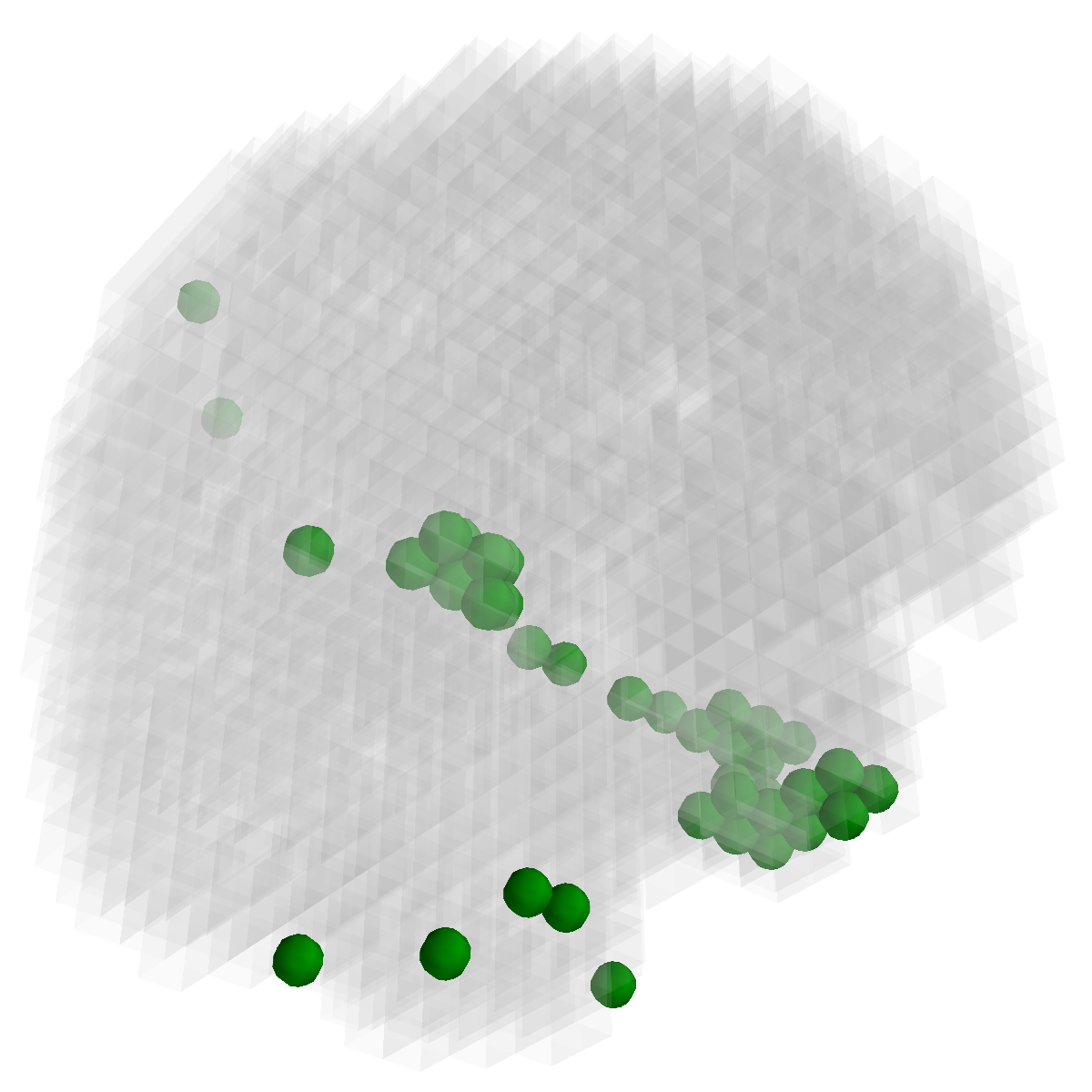}
        }
    \vspace{-3.5mm}

    \setcounter{subfigure}{0}
    \subfigure[fold 1]{
         \includegraphics[width=.33\columnwidth]{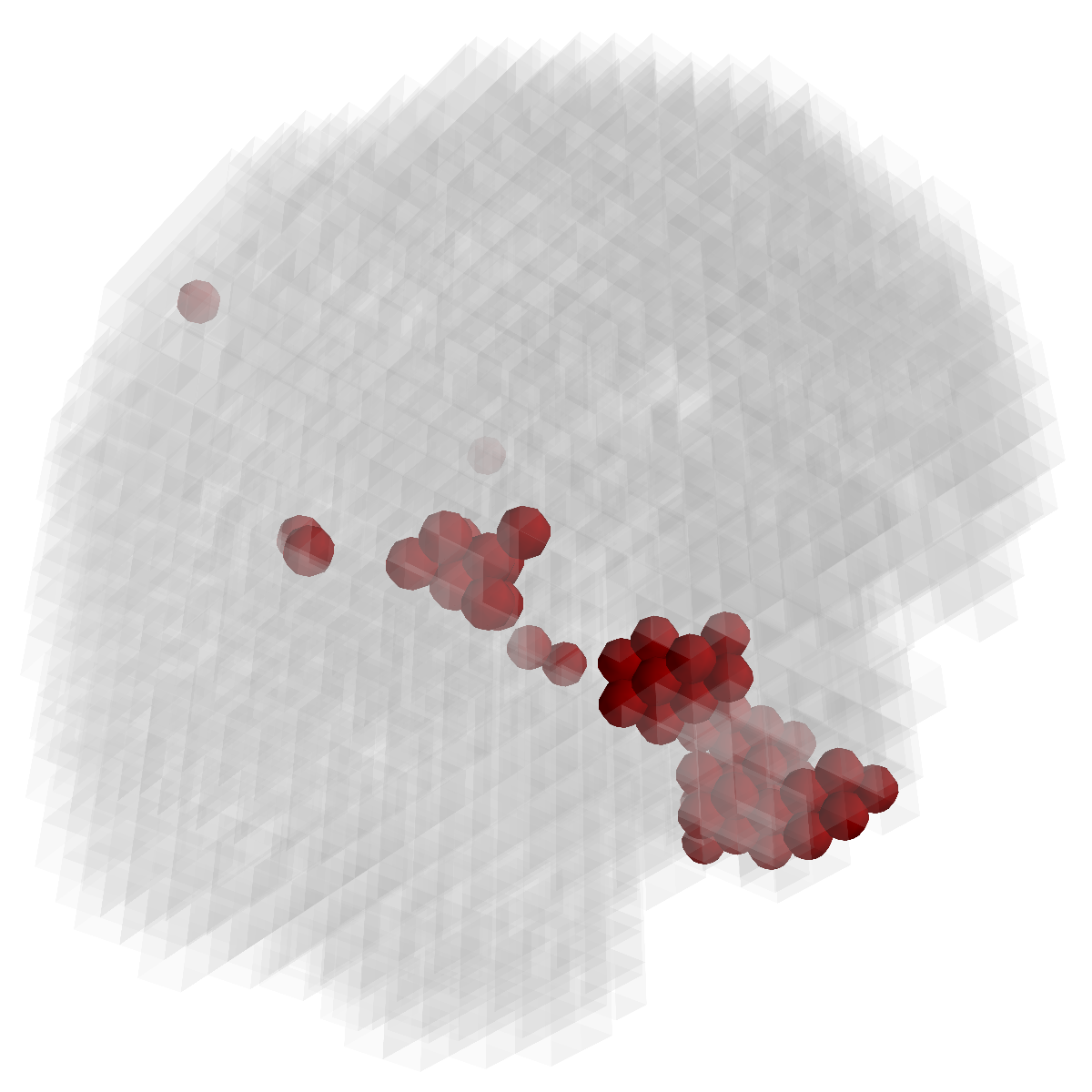}
        }
        \hspace{-7px}
    \subfigure[fold 3]{
         \includegraphics[width=.33\columnwidth]{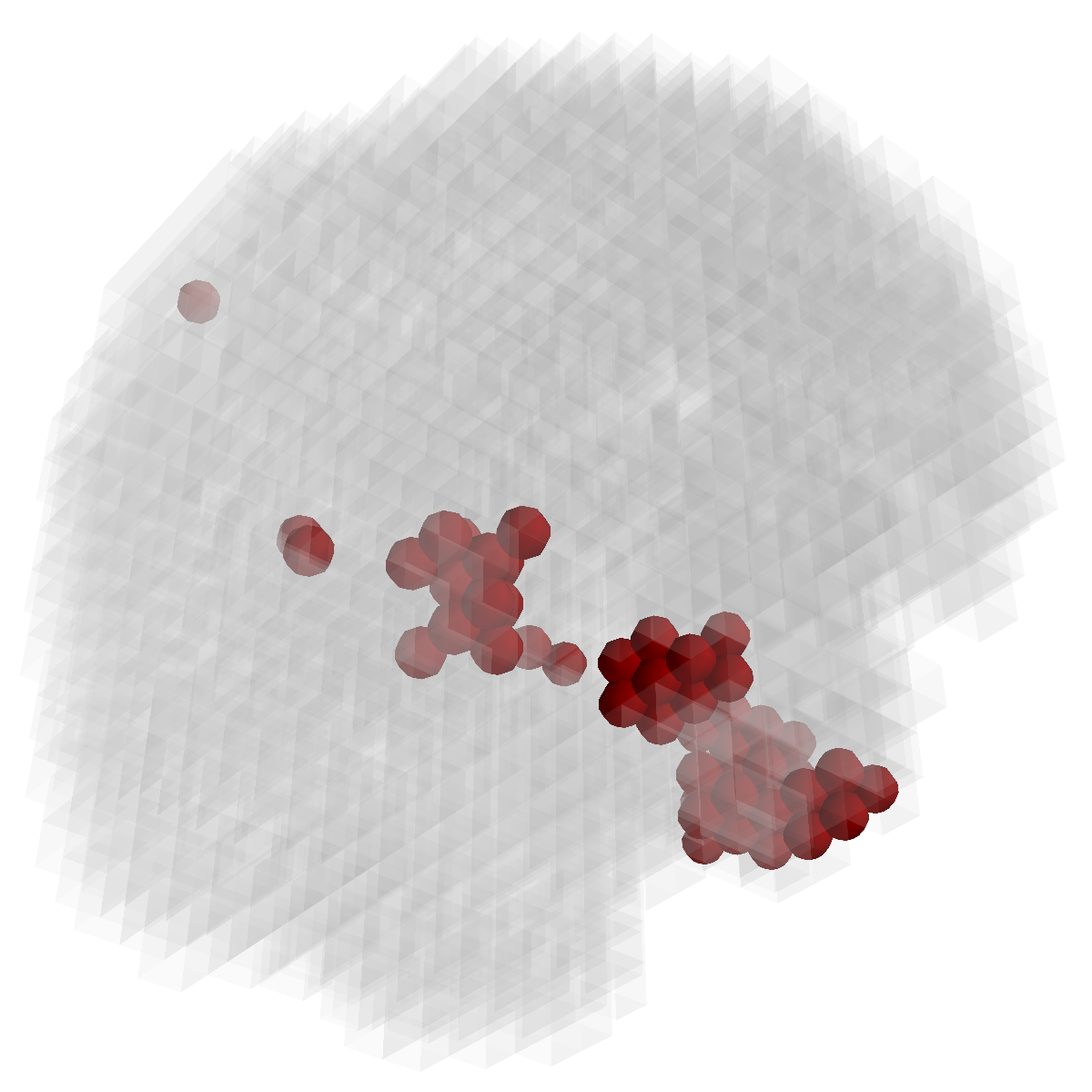}
        }
        \hspace{-7px}
    \subfigure[fold 5]{
         \includegraphics[width=.33\columnwidth]{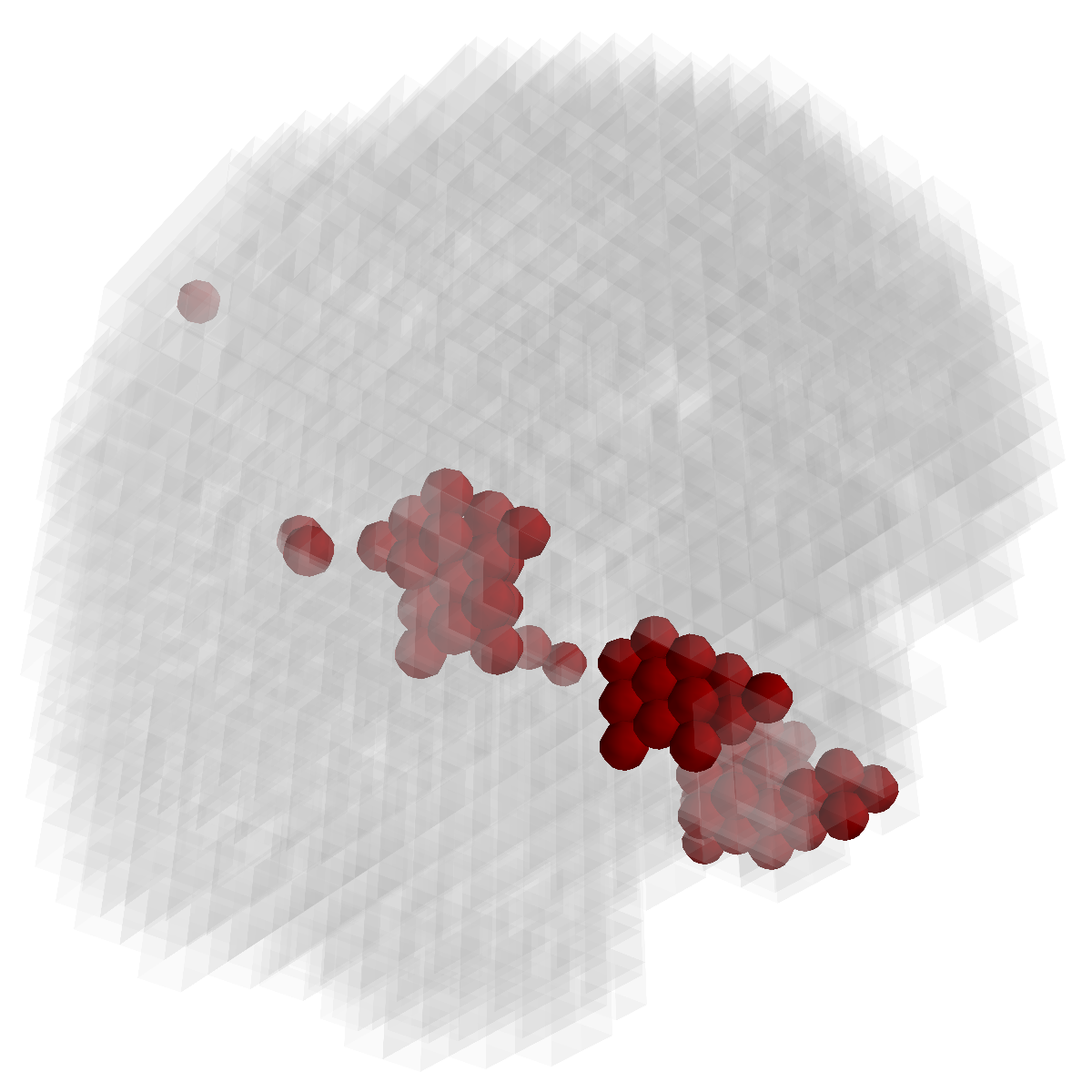}
        }
        \hspace{-7px}
    \subfigure[fold 7]{
         \includegraphics[width=.33\columnwidth]{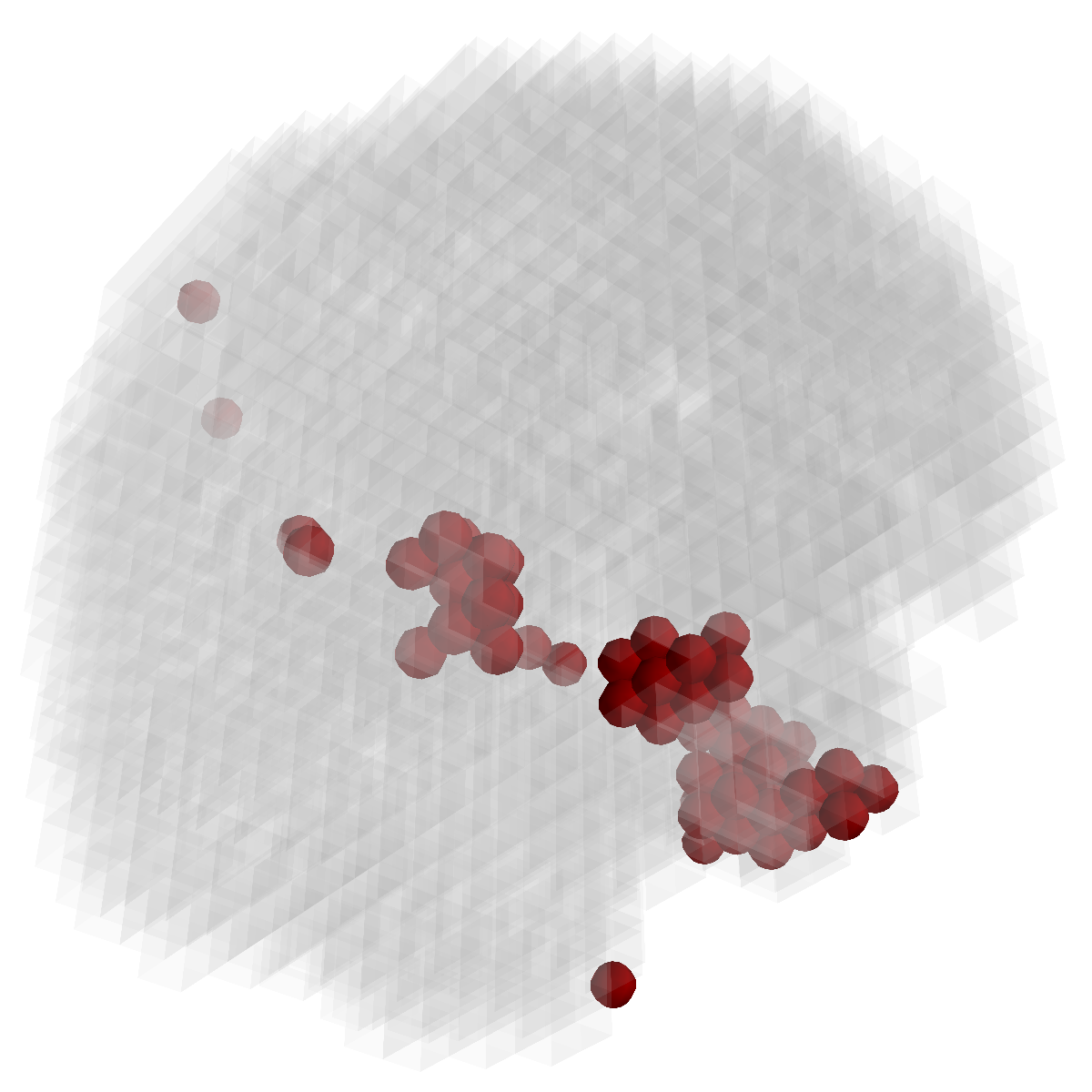}
        }
        \hspace{-7px}
    \subfigure[fold 9]{
         \includegraphics[width=.33\columnwidth]{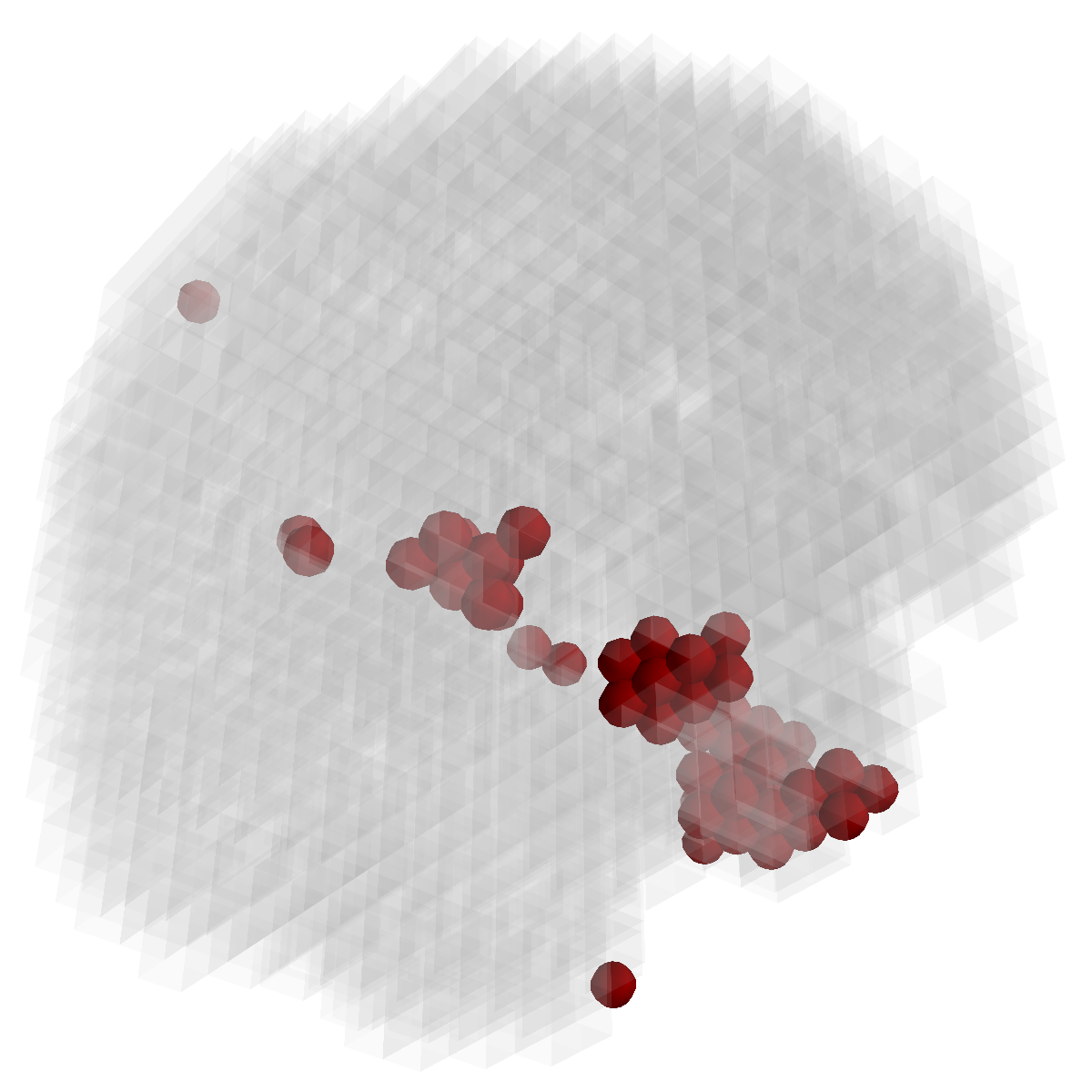}
        }
        \hspace{-7px}
    \subfigure[overlap]{
         \includegraphics[width=.33\columnwidth]{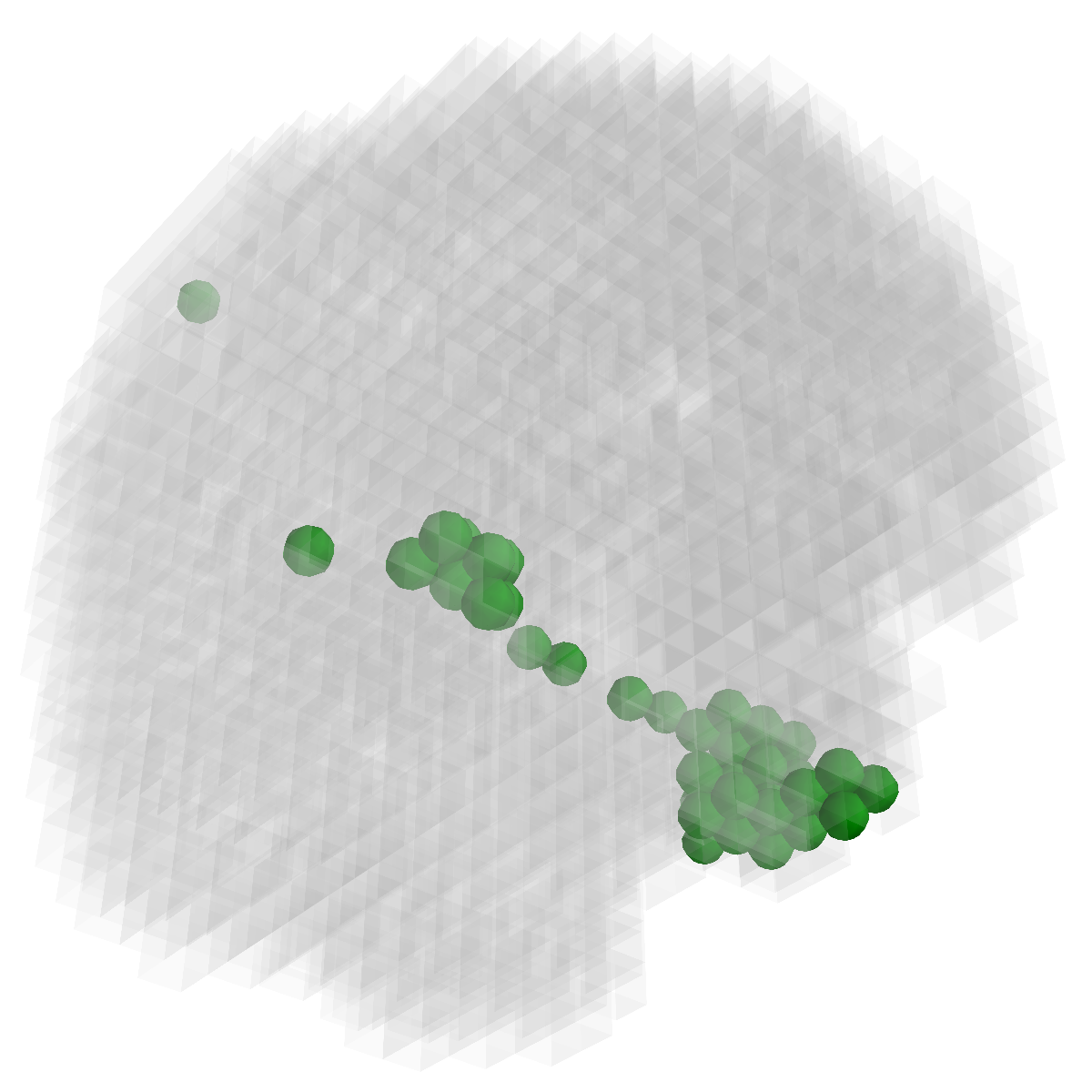}
        }
	\caption{\small{Stability of selected voxels across different folds of the cross-validation. The results of 5 different folds are shown in (a)-(e). The voxels with positive $\beta$ are in brown, negative ones are in blue. The common/overlapped voxels selected in all 10 folds are shown in green (f). The top row illustrates  voxels selected by the lasso model, the mid row illustrates those of GFL and the bottom row shows those of  $n^2$GFL.}}
\label{fig_GSR2}
\end{figure*}

\textbf{Classification Accuracy.} 10-fold cross-validation (CV) evaluation is applied and the classification accuracy for all tasks are summarized in Tab. \ref{tab:accr}. Under exactly the same experiment setup, we compare $n^2$GFL with the state-of-the-art classifiers: logistic regression (LR), SVM, sparse models e.g. the lasso and its graph Laplacian structured variants, i.e. the LapL, the unconstrained GFL \cite{bo2014gfused}, and the ``MLDA" model \cite{dai2012discriminative}, which applies a variant of Fisher Discriminant Analysis after univariate feature selection (via T-test). For each model, we used grid-search to find the optimal parameters respectively. Note that our accuracies may not be superior to the recent work \cite{liu2014inter}, the main reason is that in \cite{liu2014inter}, multi-modality data (including PET and sMRI data) are used. Nevertheless, Tab. \ref{tab:accr} demonstrates that $n^2$GFL outperforms all the other models using only voxel-based sMRI data.

\textbf{Feature selection.}  For each task, the selected features are those whose $\beta$ are not zero . In Figure \ref{fig_GSR1}, the result of 30ADNC is used to illustrate the feature selection by different models (using the parameters at their best accuracy). As shown, the selected voxels by both GFL and $n^2$GFL cluster into several spatially connected regions, whereas those of lasso and T-test/MLDA scatter around. Also, as mentioned before, the LapL tends to select much more voxels than necessary due to the $l_2$ regularization. Moreover,  the selected voxels by GFL and $n^2$GFL are concentrated in Hippocampus, ParaHippocampal gyrus (which are believed to be the early damaged regions). On the other hand, the lasso and T-test/MLDA either select less lesion voxels or select probably noisy voxels not in the early damaged regions.

\textbf{Feature Stability.}  In Figure \ref{fig_GSR2}, we show the selected voxels across different folds of CV\footnote{Here, parameters were determined by accuracy.  Similar results were observed using parameters producing same level of sparsity.}. As shown, the selected voxels by lasso vary much across different folds, whereas the selected voxels by GFL are more stable. However, by assuming the positive correlation between the features and the disease labels in $n^2$GFL, we further increase the stability. To quantitatively evaluate the stability gain, we denote the variables of the $k$th fold of CV as $\boldsymbol\beta(k)$. We introduce two measurements here. In \cite{Yu13}, the \textbf{Estimation Stability (ES)} is proposed to measure the stability of the estimation
\begin{equation}
    ES = {\sum_{k=1}^{K}{\Vert \mathbf{X}\boldsymbol\beta(k)-\mathbf{X}\bar{\boldsymbol\beta} \Vert_2^2}}/{K \Vert \bar{\boldsymbol\beta} \Vert_2^2},
\end{equation}
where $\bar{\boldsymbol\beta} = {\sum_{k=1}^{K}{\boldsymbol\beta(k)}}/{K}$. It is shown in \cite{Yu13} that ES is a fair measurement of the estimation stability. To further understand the stability of feature selection, we also extend the Dice coefficient \cite{dice1945measures} to multiple sets and apply the \textbf{multi-set Dice Coefficient (mDC)} as a measurement. We denote set $S(k)= \{i:\beta_i(k) \neq 0\}$ and define mDC as
\begin{equation}
mDC =  {K\#(\cap_{k=1}^{K}{S(k)})}/{\sum_{k=1}^{K}{\#({S(k)})}},
\end{equation}
where $\#$ is the number of elements in a set.
In Tab. \ref{tab:stab}, both measurements quantitatively suggest $n^2$GFL obtains much more stable voxels due to the consideration of the correlation between the features and the disease labels \footnote{We notice that, in \cite{bo2014gfused}, the stability is computed using the top 50 positive voxels because these voxels are believe to be the most atrophied ones. By computing the stability of all non-zero voxels, the mDC of GFL drops around $30\%$. This clearly shows that the instability is caused largely by the undesirable voxels that disagree with the correlation prior (those scattered blue voxels in the mid row).}.

\begin{table}[t]
\caption{Stability comparison of the models.}
\label{tab:stab}
\begin{center}
\begin{tabular}{p{100pt}p{30pt}p{30pt}p{30pt}}
\hline
     & lasso & GFL & $n^2$GFL \\
\hline
ES (smaller is better)      & 0.035  & 0.033  & \textbf{0.022}    \\
mDC (larger is better)      & 0.267 & 0.374 & \textbf{0.644}    \\
\hline
\end{tabular}
\end{center}
\end{table}

\section{Conclusions}
\label{sec:concl}

In this paper, we explore the nonnegative generalized fused lasso model to address an important problem of neuroimage analysis, i.e. the stability of feature selection. Experiments show that our model greatly improves the stabilities of feature selection over existing methods for brain image analysis. Although $n^2$GFL is applied to the diagnosis of AD problem, it can be applied to solve more general problems. Moreover, we believe that the theoretical points made here e.g. nonnegative FISTA, soft-thresholding and the conic dual of TV, provide motivation for future work of general interest.

\section{ Acknowledgments}
This work was supported in part by Natural Science Foundation of China (NSFC) grants 973-2015CB351800, NSFC-61272027, NSFC-61231010, NSFC-61121002 and NSFC-61210005.

\bibliographystyle{aaai}
\bibliography{ref}

\end{document}